%% file: main.tex
\pdfoutput=1
\documentclass[aps,prx,groupedaddress]{revtex4-2}
\usepackage{microtype}
\usepackage{graphicx}
\usepackage{hyperref}
\usepackage{amsmath}
\usepackage{amssymb}
\usepackage{mathtools}
\usepackage{amsthm}
\usepackage{xfrac}
\usepackage{dsfont}
\usepackage{subcaption}

\DeclareMathOperator{\unif}{Unif}
\DeclareMathOperator{\cov}{Cov}
\DeclareMathOperator{\var}{Var}
\newcommand{\scprod}[2]{\left\langle #1, #2 \right\rangle}

\theoremstyle{plain}
\newtheorem{result}{Result}[section]
\newtheorem{proposition}{Proposition}[section]
\newtheorem{lemma}{Lemma}[section]
\newtheorem{conjecture}{Conjecture}[section]

\begin{document}
\title{Specialization of softmax attention heads: \\ insights from the high-dimensional single-location model.}

\author{Margarita Sagitova}
\author{\includegraphics[height=0.79em]{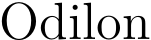} Duranthon}
\author{Lenka Zdeborová}
\email{firstname.secondname@epfl.ch}
\affiliation{%
Statistical physics of computation laboratory,\\
École Polytechnique Fédérale de Lausanne, Switzerland
}%

\title{Specialization of softmax attention heads: insights from the high-dimensional single-location model}

\begin{abstract}
  Multi-head attention enables transformer models to represent multiple attention patterns simultaneously. Empirically, head specialization emerges in distinct stages during training, while many heads remain redundant and learn similar representations. We propose a theoretical model capturing this phenomenon, based on the multi-index and single-location regression frameworks. In the first part, we analyze the training dynamics of multi-head softmax attention under SGD, revealing an initial unspecialized phase followed by a multi-stage specialization phase in which different heads sequentially align with latent signal directions. In the second part, we study the impact of attention activation functions on performance. We introduce the Bayes-softmax attention, which achieves optimal prediction performance in this setting.
\end{abstract}

\maketitle

\section{Introduction}

Multi-head attention is a central architectural ingredient of modern transformer models, enabling multiple attention patterns to coexist within a single layer. Empirical analyses show that attention heads do not all develop simultaneously during training: instead, new specialized heads emerge in distinct stages, often accompanied by sharp changes in loss or behavior suggestive of phase transitions \cite{chen24chutes}, with qualitatively new attention behaviors appearing as training progresses \cite{clark2019does,hoogland24etapesAttention,wang25rllc,chen24chutes,tigges24circuits}. At the same time, a substantial fraction of heads in trained models remain redundant and can be removed with little impact on performance \cite{michel2019sixteenheadsreallybetter,voita2019multiheadspecialized}. These observations raise a natural theoretical question: what drives staged head emergence, phase-transition-like behavior, and persistent redundancy in multi-head attention?

Recent theoretical work has begun to elucidate training dynamics and head specialization in simplified attention models. In in-context learning (ICL) of linear regression, linear attention exhibits saddle-to-saddle dynamics in which heads sequentially align with covariance eigenmodes of the data distribution \cite{saxe25iclLinear}, while multi-head softmax attention in multi-task linear regression displays short emergence phases during which different heads align with distinct tasks and converge to optimal predictors \cite{chen24icl}. These analyses establish in-context linear regression as a first solvable setting in which head specialization can be studied from first principles. We provide a complementary view on the specialization behavior by considering a model where multi-head attention itself constitutes the sole predictive mechanism and head outputs are uniformly aggregated. Rather than aiming for a more general model, our goal is to contrast solvable settings in order to identify which aspects of head specialization are shared across models and which depend on architectural or statistical details. 

In this work, we study head specialization in a controlled high-dimensional setting by specifying both a probabilistic data model and a minimal attention architecture trained on it. We consider a synthetic sequence-to-token task where the token to be recovered carries a structured signal generated from a multi-index latent model, while all other tokens contain pure noise. We learn this task with a multi-head softmax attention layer trained by stochastic gradient descent (SGD), whose head outputs are uniformly aggregated so that attention itself constitutes the sole predictive mechanism.

Our setting also connects to classical committee-machine and multi-index models: the multi-head architecture plays the role of a committee in which each head acts as a unit with its own parameters, and specialization corresponds to symmetry breaking among these units under SGD \cite{seung1992statistical,saad1995dynamics,saad1995line,engel2001statistical}. It is likewise related to recent high-dimensional analyses of SGD in multi-index and sequence single-index models \cite{aubin2018committee,goldt2019dynamics,arous2021online,arous2023highdimensionallimittheoremssgd,abbe23leap,arnaboldi2025asymptotics}. Our model is also closely related to recent solvable attention and sequence regression frameworks \cite{marion24slr,troiani2025fundamental,duranthon25slr,Barnfield25,dohmatob25slr}, which analyze high-dimensional training dynamics of simplified attention architectures from first principles.

Within the considered setting, in the high-dimensional limit, the evolution of the attention parameters reduces to a low-dimensional system of order parameters tracking head alignments with the latent signal structure, which captures the full training trajectory. 
This analysis reveals a two-stage learning dynamics: an initial fast phase where heads develop a common component aligned with the easiest (mean) signal direction, followed by a slower specialization phase in which different heads diverge and align with additional latent directions of the signal, eventually reaching stable specialized configurations. The hierarchy of these specialization events is governed by the latent signal structure, leading to sequential acquisition of increasingly subtle components (see also \cite{saxe19specialization} for stage-wise learning in deep linear networks).

We further study how attention activation functions, allowing some heads to be effectively deactivated, influence the model performance. In our setting, redundant or poorly specialized heads can inject persistent variance, making attention normalization a central modeling ingredient. While value or readout transformations can partially mitigate this effect, they act in an input-independent manner and therefore cannot fully suppress noise from redundant heads. By contrast, activations that enable head deactivation provide an input-adaptive mechanism for controlling redundancy. Alternative attention activations or extra sink tokens have been studied mainly to prevent over-focusing on irrelevant tokens \cite{kaul24softmax1,darcet24registres,xiao24eviers}. Here, we analyze such activations as tools for controlling head redundancy and specialization. Finally, \cite{qiu2025gatedattentionlargelanguage} showed that head-specific gating of attention outputs enhances performance, motivating the deactivation mechanisms without additional parameters.

Our main contributions are as follows.
\begin{itemize}
\item[(i)] We introduce a high-dimensional probabilistic framework for training multi-head softmax attention in a sequence-to-token regression task, where attention itself constitutes the sole predictive mechanism, enabling an exact characterization of learning dynamics under SGD.
\item[(ii)] We derive a closed system of equations tracking the full evolution of head alignments along the SGD trajectory, and show that training exhibits a fast unspecialized phase followed by a slower hierarchy of specialization events governed by the latent signal structure.
\item[(iii)] We analyze the effect of attention normalization on head redundancy, proving that standard softmax is generically suboptimal in this setting. We introduce the Bayes-softmax attention, which, in our setting, reaches the Bayes-risk and prescribes the optimal number of heads and the way to normalize them.
\end{itemize}

\textbf{Notations. }
For a positive integer $n$ we write $[n]=\{1,\ldots,n\}$. We use $\delta_{i,j}$ to denote a Kronecker delta that evaluates to $1$ if $i=j$ and $0$ otherwise. For a vector $v$, $v^{\odot 2}$ is the element-wise square, $\mathrm{Diag}(v)$ the diagonal matrix which diagonal is $v$. For a positive integer $n$, $\mathds{1}_{n}\in\mathbb R^n$ is the vector filled with ones, $I_n$ the $n\times n$ identity matrix, $\mathcal S_+^n$ is the set of positive symmetric $n\times n$ matrices. For a matrix $A$, $A_i$ is its $i$-th row and $A_{:i}$ is its $i$-th column. $||\cdot||_2$ is the $L_2$ norm of a vector, $||\cdot||_F$ is the Frobenius norm of a matrix. We denote a Gaussian law centered at $\omega$ with covariance $V$ as $\mathcal N(\omega,V)$.

\section{Task and data model}

We denote a sequence $X\in\mathbb R^{L\times D}$ made of $L$ tokens of dimension $D$. We choose a relevant token $X_\epsilon$ by setting a hidden index $\epsilon\in[L]$. This models the fact that, in transformers, the attention is often used to focus and extract one particular token that is needed by the subsequent layers. The relevant token is distinguished from the other tokens by a planted hidden spike $\hat k\in\mathbb R^D$. We emphasize that the latent $\epsilon$ and $\hat k$ are different in every sequence/context. Hence, their recovery is akin to a toy in-context learning (ICL) task.

\subsection{Probabilistic model of data}
\label{sec:data_distribution}
Inspired by \cite{marion24slr,troiani2025fundamental}, we introduce a probabilistic data model for this task, that is amenable to analysis in the high-dimensional limit. To define it, we begin by drawing $F$ hidden spikes
\begin{align}
k_f^*\sim\mathcal N\left(0,D^{-1}I_D\right)\quad \mathrm{for\ }f\in [F] \ .
\end{align}
The hidden spikes are common to all the sequences. In the following, we call them directions or features in an interchangeable manner. Each sequence $X\in\mathbb R^{L\times D}$ is then sampled as follows. The index $\epsilon$ of the relevant token is first taken uniformly at random over $\{1,\ldots,L\}$. We take weights $\theta\in\mathbb R^F$ for the hidden directions/features according to a distribution $P_\theta$. An effective sequence-dependent direction $\hat k\in\mathbb R^D$ is then taken as $\hat k = \sum_f^F\theta_fk_f^*$.
The tokens and the label $y$ are
\begin{align}
\label{eq:modelX}
X_\ell\sim\mathcal N\left(\delta_{\ell,\epsilon}\hat k, I_D\right)\quad \mathrm{for\ }\ell\in [L]\ ,\quad y=X_{\epsilon}\ .
\end{align}
The goal is, given $X$, to extract the relevant token $y$.
In the following, we will consider different possible choices for $P_\theta$, mainly focusing on:\\
\textbf{Flipping spike.} For $F=2$ and signal strengths $\nu_1>0, \nu_2>0$, $\theta_1=\sqrt\nu_1$ is a constant direction while $\theta_2\sim\mathrm{Unif}(\{-\sqrt\nu_2,+\sqrt\nu_2\})$ alternates. For $F>2$ and signal $\nu>0$, we take $\theta\sim\mathrm{Unif}(\{\sqrt\nu e_f\}_{f\in[F]})$ with $\{e_f\}_{f\in[F]}$ the canonical basis of $\mathbb R^F$.\\
\textbf{Non-isotropic Gaussian.} For any $F$ and two signal strengths $\nu_1\geq\nu_2>0$, we take $\theta_f\sim\mathcal N(0,\tilde\nu_f)$ for all $f$, with $\tilde\nu_1\geq\tilde\nu_2\geq\ldots\geq\tilde\nu_F$ linearly scaled between $\nu_1$ and $\nu_2$. This allows for different signal strengths for the different features. The case $\nu_1=\nu_2:=\nu$ corresponds to the \textbf{isotropic Gaussian}.\\
\textbf{MNIST semi-realistic data.} We empirically show that our results extend beyond the above data model by considering $\hat k$ drawn from the MNIST dataset. Further details are given in Appendix \ref{secApp:mnist}.

\subsection{Learning with attention}
We consider the class of estimators $\mathcal F_\sigma=\{\hat y_{\sigma,k,b,v}:\mathbb R^{L\times D}\to\mathbb R^D\}_{k\in\mathbb R^{H\times D},b\in\mathbb R^H,v\in\mathbb R}$ where for $H$ vectors $k_1,\ldots,k_H\in\mathbb R^D$, $H$ scalar biases $b_1,\ldots,b_H$, a scalar $v$ and an activation function $\sigma:(\mathbb R^{H\times L},\mathbb R^H,\mathbb R,[H])\to\mathbb R^L$, the function $\hat y_{\sigma,k,b,v}$ is defined by
\begin{align}
\label{eq:attention_model}
    \hat y_{\sigma,k,b,v}(X) = \frac{1}{H}\sum_h^H\sigma(\chi, b, v; h)^TX\ ,\quad
    \chi_h = X k_h\in \mathbb{R}^L\quad \mathrm{for\ }h\in [H]\ .
\end{align}
As to the choice of the activation function $\sigma$, we focus on the three following cases:\\
\textbf{Softmax,} the standard choice. For a head $h$ it is defined by $\sigma(\chi, b, v; h)_\ell=e^{\chi_{h\ell}} / \sum_{\ell'=1}^L e^{\chi_{h\ell'}}$. It does not depend on $b$ or $v$, and the heads only interact by uniformly aggregating their outputs.\\
\textbf{Softmax-1,} introduced in \cite{kaul24softmax1}, defined by $\sigma(\chi, b, v; h)_\ell=ve^{\chi_{h\ell}} / \big(e^{b_h} + \sum_{\ell'=1}^L e^{\chi_{h\ell'}}\big)$. Notice that the original article takes $e^{b_h}=1$, which is equivalent to our formulation up to additional head-dependent biases in the attention scores. $\sum_\ell\sigma(\chi, b, v; h)_\ell$ can become smaller than 1, thus ``deactivating'' some heads.  We allow a global rescaling by a factor $v$ to compensate inactive heads.\\
\textbf{Bayes-softmax} (or B-softmax) which normalizes each head by the output of the $H$ heads. It does not depend on $v$, but we allow for head-dependent biases. The form is, as motivated later, given by
\begin{align}
\sigma(\chi, b, v; h)_\ell=\frac{e^{\chi_{h\ell}+b_h}}{\frac{1}{H}\sum_{h'}^H\sum_{\ell'=1}^L e^{\chi_{h'\ell'}+b_{h'}}}\, .
\end{align}

This model can be viewed as a simplification of a cross-attention module where a query embedding is independent of the input sequence. Consider a single head attention layer with activation $\sigma$, scaling $\lambda$ (which in practice is usually set to $1/\sqrt{D}$), key and query matrices $K, Q\in \mathbb{R}^{D\times p}$, value matrix is identity $V=I_{D}$, and let $X_\text{query}$ be a query embedding independent of the input sequence $\mathrm{Attn}_{\sigma, K, Q, V}(X) = \sigma\left(\lambda(XK)(Q^TX_\text{query})\right)^TXV = \sigma\left(\lambda \sum_{i=1}^pa_i(XK_{:i})\right)^TX$,
where we denoted $Q^TX_\text{query} = a\in\mathbb{R}^p$ and $K_{:i}\in\mathbb{R}^D$ for $i=1, \ldots, p$ are the columns of the key matrix $K$. We get attention scores that linearly depend on the input sequence, with the key-vector $k = \lambda\sum_{i=1}^p a_iK_i$. Such attention is used, for example, in the decoder of DEtection TRansformer \citep{carion2020endtoendobjectdetectiontransformers}, where the input sequence represents encoded image, and the query represents an object to detect.

The assumption about the value matrix being identity will be relaxed in section \ref{sec:activations} on head deactivation where we will further consider the activation \textbf{softmax-v}, a softmax with non-uniform aggregation by learned value weights $v\in\mathbb R^H$, defined by $\sigma(\chi, b, v; h)_\ell=v_he^{\chi_{h\ell}} / \sum_{\ell'=1}^L e^{\chi_{h\ell'}}$.

The estimator $\hat y$ is trained by SGD. The weights $k_h$ are initialized at random $k_h^{(0)}\sim\mathcal N(0,\eta^2 D^{-1}I_D)$ with $\eta\in\mathbb R^+$ independent of $D$ controlling the initial norm. The biases are initialized to 0 and $v$ to 1. At time $t$, the weight update is
\begin{align}
\label{eq:SGD}
\begin{split}
k_h^{(t+1)}=k_h^{(t)}-\gamma\nabla_{k_h^{(t)}}\mathcal L^{(t)}, \quad
b_h^{(t+1)}&=b_h^{(t)}-\gamma\nabla_{b_h^{(t)}}\mathcal L^{(t)} \quad\mathrm{for\ }h\in [H]
\end{split}
\end{align}
and $v^{(t+1)}=v^{(t)}-\gamma\nabla_{v^{(t)}}\mathcal L^{(t)}$, where $\gamma>0$ is the learning rate and
\begin{align} \label{eq:empirical_loss}
\mathcal L^{(t)}=\frac{1}{N_b}\sum_{\mu=1}^{N_b}\frac{1}{D}||y^{\mu,(t)}-\hat y_{\sigma,k^{(t)},b^{(t)},v^{(t)}}(X^{\mu,(t)})||_2^2
\end{align}
is the empirical loss over a batch of $N_b$ sequences $\{(X^{\mu,(t)},y^{\mu,(t)})\}_{\mu\in [N_b]}$ drawn iid according to the model eq.~\eqref{eq:modelX}. We set $N=tN_b$ the total number of samples. Notice that, because of the independence of the batches, on average, the estimator minimizes the population loss
\begin{align}
\label{eq:population_loss}
\mathcal E_\sigma(k,b,v)=\frac{1}{D}\mathbb E_{X,y}\left[||y-\hat y_{\sigma,k,b,v}(X)||_2^2\right]\ .
\end{align} 

\section{Dynamics of the heads}
\label{sec:dynamics}

We consider the limit of large embedding dimension $D\to\infty$ and constant sequence length, number of spikes, number of heads, signal strength, and initialization, i.e., $L, F, H, ||\theta||_2, \eta=\Theta(1)$. A first consequence of this limit is that the population loss $\mathcal{E}_\sigma$ eq.~\eqref{eq:population_loss} can be expressed in terms of a few \emph{order parameters} (or \emph{sufficient statistics}).

\begin{proposition}[Reparametrized loss]
\label{res:paramLoss}  
The loss of the attention $(k,b,v)\mapsto \mathcal{E}_\sigma(k,b,v)$ can be reparametrized as a function $(m,r,b,v)\in(\mathbb R^{H\times F},\mathcal S_+^H,\mathbb R^H,\mathbb R)\mapsto \tilde{\mathcal{E}}_\sigma(m,r,b,v)$ of the following order parameters, for $h,h'\in [H]$, $f,f'\in [F]$: 
\begin{align}
    m_{hf}=(k_h)^\top k_f^*, \quad q_{hh'} =(k_h)^\top k_{h'}, \quad
    p_{ff'} = (k_f^*)^\top k_{f'}^*, \quad r = (q-mp^{-1}m^\top)^{\sfrac{1}{2}}
\end{align}
Let $\epsilon\sim\unif(\{1,\ldots,L\})$, $\theta\sim P_\theta$ and conditionally on $\epsilon$ and $\theta$, $\chi_\ell^*\sim\mathcal N(\delta_{\ell,\epsilon}\theta,I_F)$ and $\xi_\ell\sim\mathcal N(0,I_H)$ for $\ell\in[L]$. The reparametrized loss is
\begin{align}
& \tilde{\mathcal{E}}_\sigma(m,r,b,v) = \mathbb E_{\epsilon,\theta,\chi,\xi}\left[\sum_\ell^L\Big(\delta_{\ell,\epsilon}-\frac{1}{H}\sum_h^H\sigma(\chi, b, v; h)_\ell\Big)^2\right], \\
& \mathrm{where\ }\chi_{h\ell} = \sum_f^Fm_{hf}\chi_{f\ell}^*+\sum_{h'}^Hr_{hh'}\xi_{h'}\ ,\quad h\in[H], \ell\in[L]\ . \nonumber
\end{align}
\end{proposition}
A proof is given in Appendix~\ref{secApp:low_dim_sgd}. Here $m_{hf}$ quantifies the alignment between the head $h$ and the hidden direction $f$, while $q$ is the overlap between the heads. $r$ is the amplitude of the components that are orthogonal to the spikes; $r=0$ whenever $k$ is in the span of $k^*$. In our model, $p\approx I_F$ because the hidden directions are independently sampled from a high-dimensional normal distribution.

The motivation to introduce these order parameters is that the learning dynamics of the attention can be expressed in a closed-form over $m$ and $r$. We consider the limit of gradient flow on the population loss, obtained by taking $\gamma N_b^{-1}=o(D^{-1})$ together with $N_b\geq 1$ and $\gamma=\mathcal O(1)$, as detailed e.g. in \cite{arnaboldi2024online}. In particular, this encompasses the two following cases: i) large batches $N_b=\Theta(D)$ and small learning rates $\gamma=o(1)$; and ii) constant-size batches $N_b=\Theta(1)$ and vanishing learning rates $\gamma=o(D^{-1})$. 
This scaling allows us to derive the correspondence between the online SGD and the gradient flow, similarly to \cite{saad1995line,arous2021online,arnaboldi2024online}.

\begin{proposition}[Effective dynamics]
\label{res:sgd}
Consider the dynamics of SGD defined by eq.~\eqref{eq:SGD}. Let $\tau=\gamma t$ be the effective time. Then, in the space of the order parameters, the dynamics is given by
\begin{align}
\label{eq:dynEff}
\frac{\partial}{\partial\tau}m(\tau) = -\nabla_m\tilde{\mathcal{E}}_\sigma(m(\tau),r(\tau),b(\tau),v(\tau)), \quad
\frac{\partial}{\partial\tau}r(\tau) = -\nabla_r\tilde{\mathcal{E}}_\sigma(m(\tau),r(\tau),b(\tau),v(\tau))
\end{align}
and $\partial_\tau b=-\nabla_b\tilde{\mathcal{E}}_\sigma$ and $\partial_\tau v=-\nabla_v\tilde{\mathcal{E}}_\sigma$.
\end{proposition}
A proof of this result is given in Lemma \ref{resApp:high_dim_limit} in Appendix \ref{secApp:sampleComplexitySGD}. Similar results, on the equivalence between the dynamics of SGD in the high-dimensional space and the space of order parameters, have been derived for fully-connected two-layer neural networks in e.g. \cite{saad1995line,arous2021online,abbe23leap,arnaboldi2024online} and in \cite{arnaboldi2025asymptotics} for sequence single-index models. As a numerical check, we provide simulations in Fig.~\ref{fig:compDfiniVsFlot} and in Figs.~\ref{fig:compDfiniVsFlot_supp} and \ref{fig:compDfiniVsFlot_suppsuppsupp} in the Appendix. The agreement between the theory and the simulations is very good.

\begin{figure}[ht!]
    \centering
    \includegraphics[width=0.95\linewidth]{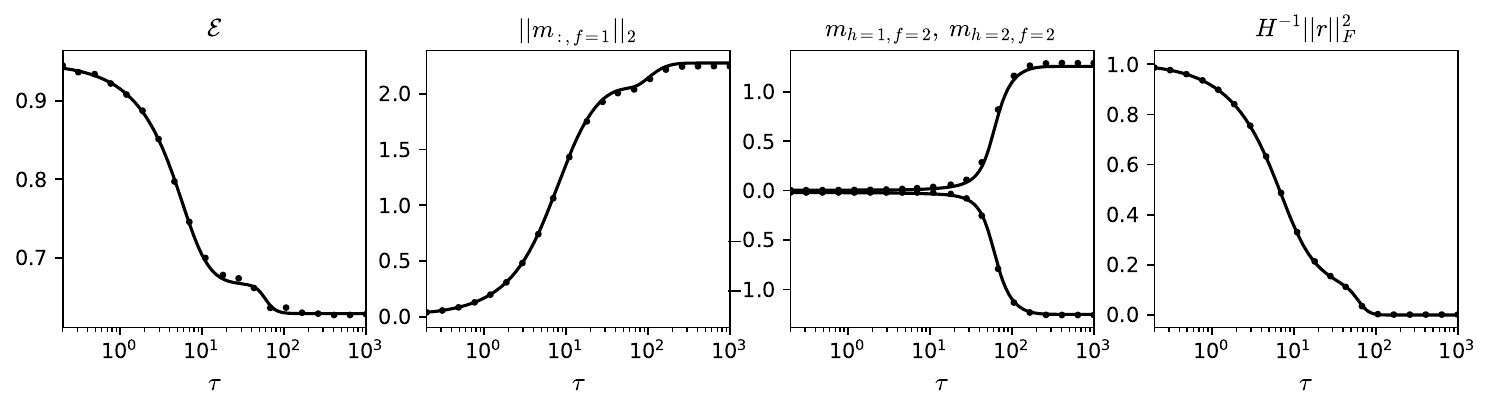}
    \caption{\label{fig:compDfiniVsFlot} Asymptotic description of the attention trained by SGD. We compare numerical simulations at finite $D=10^4$ (dots) and the theoretical description stated in proposition~\ref{res:sgd} (continuous lines). Left: loss; center left: alignment along $f=1$ the constant direction; center right: alignment along $f=2$ the ``flipping sign'' direction; right: orthogonal component. We consider sequence length $L=10$, $H=2$ heads with $\sigma$ softmax attention, $F=2$ features, $\theta$ drawn from the flipping spike distribution, with signal strengths $\nu_1=\nu_2=2$. Initialization $\eta=1$.}
\end{figure}

A consequence of Prop.~\ref{res:sgd} is the characterization of two distinct phases in the learning. At random initialization $m$ concentrates around 0. In the first phase, the heads do not specialize and move collectively towards the mean of the signal in a few time steps. The dynamics is controlled by the gradient of the loss in the direction of the mean signal.
\begin{proposition}[Unspecialized phase]
\label{res:1stPhase}
Let $\mathbb E\theta=(\mathbb E_{P_\theta}\,\theta_f)_{f\in[F]}\in\mathbb R^F$ be the mean of the feature weights and assume it does not vanish $\mathbb E\theta\neq 0$. There is a finite time $\tau^\mathrm{u}=\Theta(1)$ for which there is a $x\in\mathbb R^+, x=\Theta(1)$ such that, for all $h\in[H]$, $m_h(\tau^\mathrm{u})=x\mathbb E\theta+\mathcal O(D^{-\sfrac{1}{2}})$.
\end{proposition}
This result comes from the facts that the space of unspecialized $m$, $r$ and $b$ is invariant by the gradient flow, as stated in Lemma~\ref{resApp:invUnspMan}, and that at $m=0$ the gradient does not vanish and points towards $\mathbb E\theta$, as stated in Lemma~\ref{resApp:gradI} in Appendix~\ref{secApp:derivatives}. Learning the mean direction $\mathbb E\theta$ is fast and in total it requires $N$ larger than $\Theta(D)$ samples to recover this direction. This is depicted in Fig.~\ref{fig:compDfiniVsFlot}, where we consider the flipping sign distribution, where $\mathbb E\,\theta_1 > 0$ and $\mathbb E\,\theta_2=0$. In the 2nd panel, a time $\tau^\mathrm{u}\approx 1$ is enough to start learning the direction $f=1$. In the case when $\mathbb E\theta=0$, $m$ does not move until the next phase. The biases $b$ stay unspecialized: around $\tau^\mathrm{u}$ there is $\tilde b\in\mathbb R$ such that $b_h=\tilde b$ for all $h$, with for the B-softmax $\tilde b=0$ and for the softmax-1 $\tilde b$ is such that on average the output of the attention scales to 1. Since the space of unspecialized $m$ is invariant under the gradient flow, the specialization of the heads can occur only due to the asymmetry introduced at initialization or with SGD updates. Therefore, a different rescaling of magnetizations $m_{hf}$ is needed to study the diffusive regime of the dynamics as suggested by \cite{arous2023highdimensionallimittheoremssgd}. We derive the limiting dynamics under a learning rate scale that is slightly smaller than $\gamma N_b^{-1}= o(D^{-1})$ in Lemma \ref{resApp:sgdLimit} and it leads to the following proposition:

\begin{proposition}[Specialization phase]
\label{res:2ndPhase}
Let $P_{\bot\mathbb E\theta}$ be the projection onto the space orthogonal to $\mathbb E\theta$. Assume that $P_{\bot\mathbb E\theta}\cov(\theta)P_{\bot\mathbb E\theta}^\top$ is not fully degenerated. Assume that $||\mathbb E\theta||_2$ and $\eta$ are small enough, independently of $D$. There is a time $\tau^\mathrm{s}=\Theta(\log D)$ at which, for all $h\in[H]$, $||P_{\bot\mathbb E\theta}\,m_h(\tau^\mathrm{s})||_2=\Theta(1)$ and the heads start specializing: $||m_h-m_{h'}||_2=\Theta(1)$ for some $h\neq h'$.
\end{proposition}
Learning the directions orthogonal to $\mathbb E\theta$ and specializing the heads is slower and in total it requires $N$ larger than $\Theta(D\log D)$ samples to start specializing. As depicted in Fig.~\ref{fig:compDfiniVsFlot}, 3rd panel, time $\tau^\mathrm{s}\approx 10$ is needed at $D=10^4$, during which the alignment with the orthogonal direction and the difference between the heads plateau at 0. $\tau^\mathrm{s}$ is the time needed to escape the saddle at $P_{\bot\mathbb E\theta}m=0$; its precise value depends on the details of the initialization $P_{\bot\mathbb E\theta}m(\tau=0)$ and varies between runs. In particular, to obtain a match between the theory and the simulation one has to initialize $m$ and $r$ to their empirical values at $\tau=0$. Prop.~\ref{res:2ndPhase} is derived for small enough $||\mathbb E\theta||_2$ and $\eta$ to ensure that the specialization happens close to the initialization and can be analytically tracked. Such a hypothesis is commonly made, e.g. in \cite{saxe25iclLinear}, \cite{pesme23saddle} and references therein. Still, Prop.~\ref{res:2ndPhase} holds for larger values, as numerically shown by Figs.~\ref{fig:compDfiniVsFlot} to \ref{fig:hierarchical3D} and in Appendix~\ref{sec:training_dynamics_plots}.

These two stages are further illustrated in Fig.~\ref{fig:spec_vs_H} for the flipping sign distribution, for different $H$, where we see that in the first phase of training, all the heads align with the constant direction, and only after some time do they begin to diverge along the ``flipping sign'' direction. When increasing the number of heads, some heads specialize and diverge more from the average direction. The specialization of the heads versus the two signal strengths and versus $H$ is further numerically analyzed in Appendix~\ref{sec:specialization_vs_nu}.

\begin{figure}[t]
\centering
\begin{subfigure}[t]{0.485\textwidth}
    \centering
    \includegraphics[width=\linewidth]{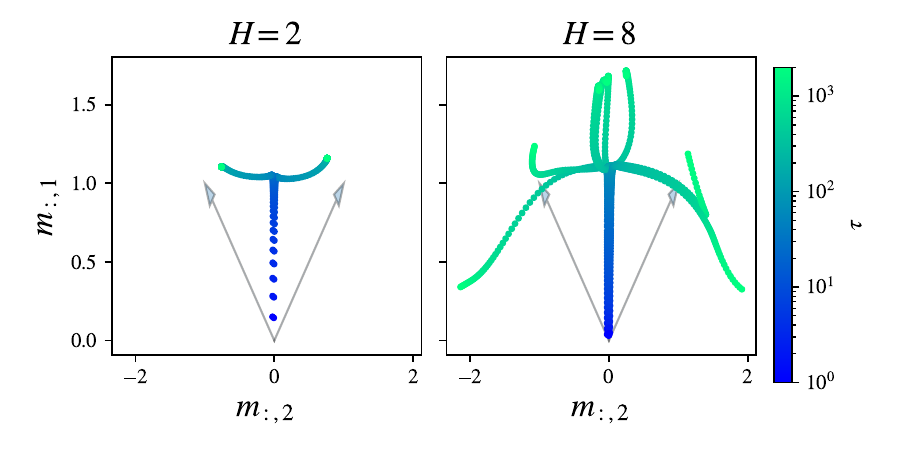}
    \caption{\label{fig:spec_vs_H}}
\end{subfigure}
\hfill
\begin{subfigure}[t]{0.485\textwidth}
    \centering
    \includegraphics[width=\linewidth]{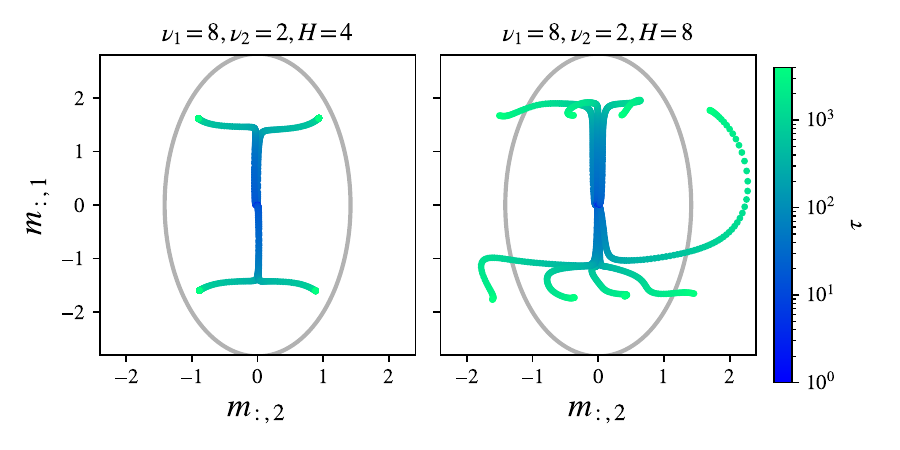}
    \caption{\label{fig:elipsoid} }
\end{subfigure}
\caption{\label{fig:f2_evolution}Evolution of the $H$ heads at $F=2$, softmax activation, $L=4$, and $\eta=1$. Left: for the flipping spike distribution, $\nu_1=\nu_2=2$, and right for the non-isotropic Gaussian distribution, $\nu_1=8, \nu_2=2$.}
\end{figure}
~\\

\textbf{Phase transition. } The (fast) specialization of the heads and the acquisition of new behaviours is often called phase transition, breakthrough or sudden emergence \cite{chen24chutes}. Our model allows us to clarify the nature of the dynamical transition around $\tau^\mathrm{s}$. The transition from unspecialized to specialized heads is truly a dynamical phase transition, in the sense of statistical physics. More precisely, we conjecture that, under the right rescaling when $D\to\infty$, the specialization time $\tau^\mathrm{s}$ concentrates to a deterministic value and the specialization transition around it is sharp. This conjecture relies on a heuristic argument and is supported by numerics detailed in Appendix \ref{secApp:phaseTransition} and Fig.~\ref{fig:phaseTransition}. It moreover extends to semi-realistic data, where $D$ is fixed, by considering the limit of $\eta\to 0$, as we show in App.~\ref{secApp:phaseTransition} Fig.~\ref{fig:phaseTransition_mnist}; similarly to the sharp transition that appears in \cite{saxe25iclLinear} with small initialization.
\begin{conjecture}[Dynamical phase transition]\label{conj:phase}
Take $\delta>0$ small enough. Define $\tau^\mathrm{s}$ as the first time $\tau$ such that there are some $h,h'$ such that the heads are specialized $||m_h-m_{h'}||_2>\delta$. Under the rescaling $\tilde\tau^\mathrm{s}=\tau^\mathrm{s}\log(\sqrt D/\eta)^{-1}$, in the limit $D\to\infty$, $\tilde\tau^\mathrm{s}$ converges to a deterministic value, independent of the realization of the data $k^*, X, y$, of the initialization $k^{(0)}$ and of $D,\eta,\delta$.
\end{conjecture}
~\\

\textbf{Sequential specialization. }
We give insights on how the specialization of the heads occurs during the specialization phase.
The following discussion is based on an analytical result Lemma~\ref{res:hessI} for the initial times of the specialization, and for the later times relies on heuristic arguments on the structure of the landscape and on a numerical integration of Prop.~\ref{res:sgd}.
\begin{result}[Sequential specialization]
\label{res:seqSpe}
During and after $\tau^\mathrm{s}$, the specialization of the heads occurs in a sequential way, via a saddle-to-saddle dynamics, learning the eigenvectors $(e_f)_{f\in[F]}$ of $\cov\theta$ from the largest eigenvalue to the smallest. Moreover, softmax and softmax-1 learn mixtures of $(e_f)_f$.
\end{result}
This is a key connection with what \cite{hoogland24etapesAttention,wang25rllc} observe in practice, where the attention first learn easy tasks such that bigram statistics and then harder such that n-grams and induction. Such sequential specialization is also described in \cite{saxe25iclLinear} for ICL of linear regression. This sequential learning is shown for the non-isotropic Gaussian distribution, where $\mathbb E\theta=0$ and all the features $f$ have different signal strengths in Figs.~\ref{fig:elipsoid} and \ref{fig:hierarchical3D}, as well as in Appendix Figs.~\ref{fig:compDfiniVsFlot_suppsuppsupp} and \ref{fig:ellipsoide_softmax}. In the case where two features have the same signal strength, they are learned at the same time, as shown in Fig.~\ref{fig:ellipsoide_softmax} top.

Compared to \cite{saxe25iclLinear,chen24icl}, we observe that each head of the attention does not focus on a single direction $(e_f)_f$. Instead, the softmax and the softmax-1 learn the $2^F$ possible combinations $\pm e_1\pm e_2\ldots \pm e_F$. This is shown on Fig.~\ref{fig:elipsoid} at $F=2$ and $H=4$ and in the inset of Fig.~\ref{fig:hierarchical3D} left at $F=3$ and $H=8$, as well as in Appendix Fig.~\ref{fig:ellipsoide_softmax}. A similar hierarchical learning is described in \cite{saxe19specialization} for deep linear fully-connected neural networks; yet notice that, contrary to this work, we do not impose any strong structure on the data $P_\theta$ since we only require anisotropy. Rather, the hierarchical learning comes from the structure of the softmax multi-head attention itself. The behaviour of the B-softmax is different: the heads tend to learn single directions $\pm e_f$, as shown in Fig.~\ref{fig:hierarchical3D} right; further insights are given by Prop.~\ref{res:optimalAttention} in the next part and an extension to semi-realistic data is shown in Fig.~\ref{fig:mnist_vsH}.

\begin{figure}[t]
    \centering
    \includegraphics[height=0.3\linewidth]{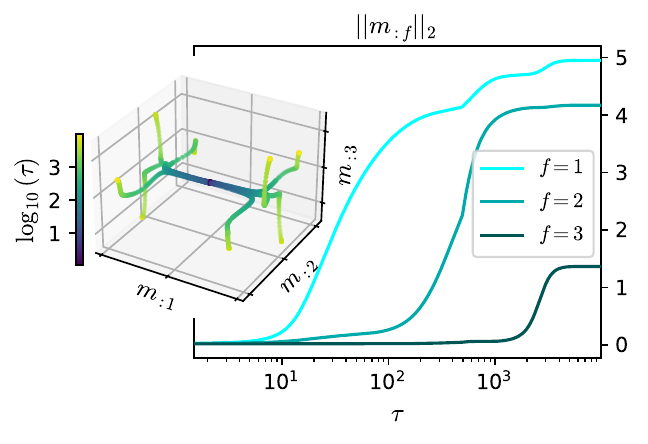}
    \includegraphics[height=0.3\linewidth]{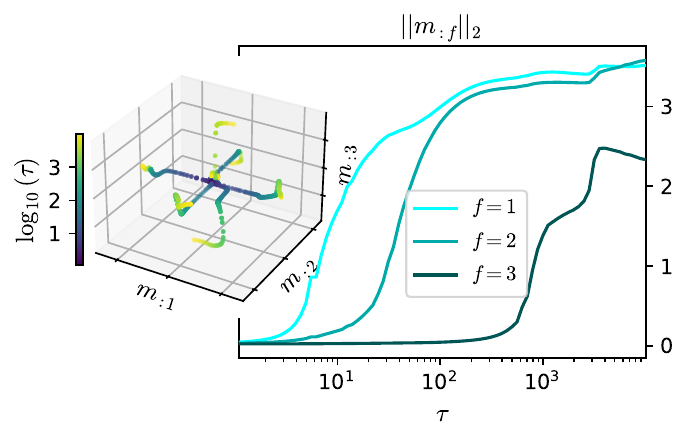}
    \caption{\label{fig:hierarchical3D} Evolution of the $H=8$ heads for the non-isotropic Gaussian distribution at $F=3$, according to Prop.~\ref{res:sgd}. Left: $\sigma$ softmax; right: $\sigma$ B-softmax. $L=5$, $\nu_1=20$, $\nu_2=1$ and $\eta=1$.}
\end{figure}

Result \ref{res:seqSpe} comes from the following analysis. The time to escape the saddle and to specialize is controlled by the curvature of the loss in the different directions: the more negatively curbed the fastest. We state the following lemma on the Hessian before the specialization time $\tau^\mathrm{s}$, when the attention is not yet specialized, and after the time $\tau^\mathrm{u}$, so $b$ and $v$ for the softmax-1 reached their unspecialized minimum. This lemma is a simplification of the Lemma~\ref{resApp:hessI} given in Appendix~\ref{secApp:derivatives}. We assume that $||\mathbb E\theta||_2$ and $\eta$ are small enough so we can expand the loss around $m\approx 0$ and at $r=0$.

\begin{lemma}[Hessian before specialization]
\label{res:hessI}
Consider $m\in\mathbb R^{H\times F}$ orthogonal to $\mathbb E\theta$ i.e. $m\mathbb E\theta=0$. Assume that $b$ is not specialized, and for the softmax-1 assume that $b$ and $v$ reached the fixed-point described by Lemma \ref{resApp:gradBVI}. The loss can be expanded in $m$ as the quadratic form
\begin{align}
& \tilde{\mathcal{E}}_\sigma(m,0,b,v) = \tilde{\mathcal{E}}_\sigma(0,0,b,v)\\
&+ \left(\mathds{1}_H\mathds{1}_H^\top\otimes(c_1^{(2)}I_F+(c_2^{(2)}+c_4^{(2)})\cov\theta)-(c_3^{(2)}+c_4^{(2)})I_H\otimes\cov\theta\right)\cdot(m,m) + \mathcal O(||m||_F^4)  \nonumber
\end{align}
with the tensorial product $\otimes$ between the spaces $\mathbb R^H$ and $\mathbb R^F$, and
$c_1^{(2)}, c^{(2)}_2, c^{(2)}_3\in\mathbb R$ strictly positive and $c^{(2)}_4\in\mathbb R$ positive for all $L\geq 3$.
\end{lemma}
The first descent directions of the magnetization $m$ in the feature space $\mathbb R^F$ are thus the eigenvectors of $\cov\theta$ with largest eigenvalue, as stated by Result~\ref{res:seqSpe}.

We can provide a more precise description of the dynamics. In the head space $\mathbb R^H$, according to Lemma~\ref{res:hessI} the first descent directions are all the directions orthogonal to $\mathds{1}_H$; that is to say the heads $m_{:f}$ tend to split evenly across each feature $f$. While Lemma \ref{res:hessI} applies at the beginning of specialization, Lemma~\ref{resApp:diff4} in Appendix~\ref{secApp:derivatives} further extends the analysis to later times. The outcome is that, assuming that the features $\bar f_1,\ldots,\bar f_n$ are already partially learned and that the feature $f$ is not yet learned, $m_{:f}$ tends to grow orthogonally to the magnetizations $m_{:\bar f_i}$ of each already-learned feature. This description corresponds to what is shown on Fig.~\ref{fig:elipsoid} and \ref{fig:hierarchical3D}.
~\\

\textbf{Later dynamics, excursions. }
The later specialization dynamics, and in particular the final split of the heads, depend on the initial condition $k^{(0)}$; it is thus challenging to accurately and entirely describe. As can be seen in Fig.~\ref{fig:elipsoid} at $H=8$, because of the stochasticity in initialization, at the beginning of the learning the heads may not split evenly along a direction $f$, even if $P_\theta(\theta_f)$ is symmetric in this direction. However, later the misplaced heads can rearrange, by performing an \textit{excursion}, so that at the end the split is even and the loss optimized. This shows that the structure of the multi-head softmax helps SGD to navigate the landscape and not to stay stuck in local minima. Moreover, in the next section Fig.~\ref{fig:compActivationFunctions_vsHNu} we provide a numerical guarantee that the Bayes-softmax reaches the global minimum of the loss whenever $H$ is large enough, which echoes the result of \cite{chen24icl} on the optimality of gradient flow.

\section{Head deactivation via alternative activation}
\label{sec:activations}

In this section, we consider the trained models and the achieved loss $\mathcal{E}^\infty_\sigma=\lim_{\tau\to\infty}\mathcal E_\sigma(k(\tau),b(\tau),v(\tau))$, for different activations $\sigma$. We start by deriving the Bayes risk on our probabilistic data model, motivating the Bayes-softmax attention; we compare the expressivity and the performances of the other activations functions and we highlight the impact of the normalization of the heads.

A natural benchmark is the Bayes risk $\mathcal{E}_\mathrm{Bayes}=\mathbb E_{k^*,X,y}||y-\hat{y}_\mathrm{Bayes}(X, k^*)||_2^2$. It corresponds to the loss of the Bayes estimator $\hat{y}_\mathrm{Bayes}$, which is the optimal estimator of $y$ in terms of population loss. We characterize it in the following proposition and derive it in Appendix~\ref{sec:BayesProofs}.

\begin{proposition}[Bayes estimator]
\label{res:BayesRisk}
Given the spikes $\{k^*_f\}_{f\in[F]}\in\mathbb R^{F\times D}$ and a sequence $X\in\mathbb R^{L\times D}$, the Bayes estimator of the label is
\begin{align}
\hat{y}_\mathrm{Bayes}(X, k^*) &= \sum_{\ell}^L \frac{\int_\theta \exp{(\hat{k}(\theta)^TX_\ell)}e^{-\frac{||\theta||_2^2}{2}}P_\theta(d\theta)}{\sum_{\ell'}^L\int_\theta \exp{(\hat{k}(\theta)^TX_{\ell'})}e^{-\frac{||\theta||_2^2}{2}}P_\theta(d\theta)}\,X_\ell\ , \quad \mathrm{where\ } \hat{k}(\theta) = \sum_f^F\theta_fk_f^*.
\end{align}
\end{proposition}

In practice, one does not have access to the spikes $k^*$ nor to $P_\theta$, and $\hat{y}_\mathrm{Bayes}$ seems to be a purely theoretical estimator. Yet, $k^*$ and $P_\theta$ can be learned and $\hat{y}_\mathrm{Bayes}$ can be interpreted as the Bayes-softmax attention, which we rewrite as
\begin{align}
\hat y_{\textnormal{B-softmax},k,b}(X) &= \sum_{\ell}^L\frac{\sum_h^H\exp{(k_h^\top X_\ell+b_h)}}{\sum_{\ell'}^L\sum_{h'}^H\exp{(k_{h'}^\top X_{\ell'}+b_{h'})}} X_\ell\ .
\end{align}
We state the equivalence between the Bayes estimator and the Bayes-softmax attention in the following proposition, which directly follows from a substitution of the given parameters into the B-softmax attention model leading to the expression of the Bayes estimator Prop.~\ref{res:BayesRisk}.

\begin{proposition}[Optimality of the Bayes-softmax attention]
\label{res:optimalAttention}
Consider some distribution $P_\theta$ with discrete support $\{\theta^h\}_{h\in[H]}$.
Then the Bayes-softmax attention with $H$ heads and parameters $k_h = \hat{k}(\theta^h)$ and $b_h = -||\theta^h||_2^2\log{P_\theta(\theta^h)}/2$ achieves the Bayes risk.
\end{proposition}
This proposition gives a prescription on the right number $H$ of attention heads: each point of the support of $P_\theta$ should correspond to a different attention head. This can be generalized to continuous distribution $P_\theta$ by discretizing the integral in Prop.~\ref{res:BayesRisk} and taking $H$ large enough to approximate it correctly. 
In Fig.~\ref{fig:compActivationFunctions_vsHNu} we can see that when the distribution of the spikes is discrete, the loss of the B-softmax model reaches a plateau to the Bayes risk when the number of heads is greater than or equal to the number of spikes. This shows that the B-softmax trained with SGD can exactly estimate the hidden parameters of the Bayes-risk and reach the optimality, whenever $H$ is large enough. For a continuous distribution, the loss does not plateaus and continues to decrease, and more heads are required to interpolate $P_\theta$ and approach the Bayes-risk. According to Fig.~\ref{fig:compActivationFunctions_vsHNu} center left, at $F=4$ for Gaussian $P_\theta$, $H=5$ is already enough to be close to the optimality; and the B-softmax does not seem cursed by the dimensionality of $P_\theta$. The same behaviour qualitatively holds for the softmax and the softmax-1, though the equivalence with the Bayes risk does not hold.

In our setting, heads that are not aligned with the signal $\hat k$ introduce noise that cannot be reduced by other means. This misalignment has to be mitigated by the architecture of the attention itself, by ``switching off'' the heads with low attention scores. While the standard softmax is not able to do so and therefore cannot reach zero error, the softmax-1 and B-softmax can deactivate heads and outperform the standard variant.

We additionally consider the \textbf{softmax-v} attention, designed to emulate the effect of value matrix of the traditional attention architecture on the deactivation of heads. In more realistic settings, relevant and irrelevant tokens may span different subspaces in the space of embeddings and can be separated by value matrices, output projections or MLPs. Yet, the attention still plays a central role, as shown by the works on the softmax-1 and attention sinks \cite{kaul24softmax1,darcet24registres,xiao24eviers}, where these mechanisms do improve the effectiveness of the standard softmax. Since in our model embeddings of both the relevant and irrelevant tokens span the whole $\mathbb{R}^D$ in an isotropic manner, we only consider the scaling effect introduced by the value matrices. The additional parameters in softmax-v model allow to partially cancel outputs of the not specialized heads, but coefficients $v_h$ cannot adapt to the input sequence. However, as stated by the following proposition, proved in Appendix \ref{sec:BayesProofs}, the ability to normalize attention scores adaptively based on the current sequence is crucial for performance in our model.

\begin{figure}[t]
    \centering
    \includegraphics[width=0.47\linewidth]{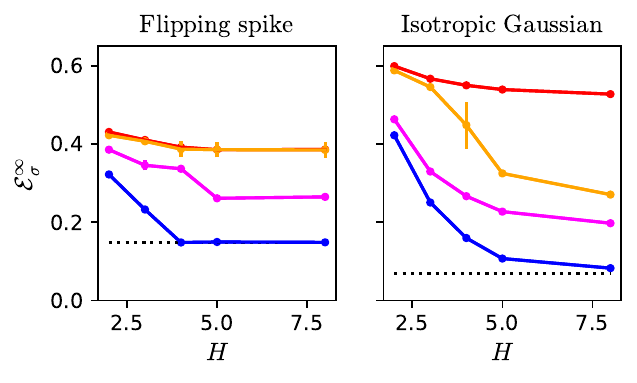}
    \hfill
    \includegraphics[width=0.47\linewidth]{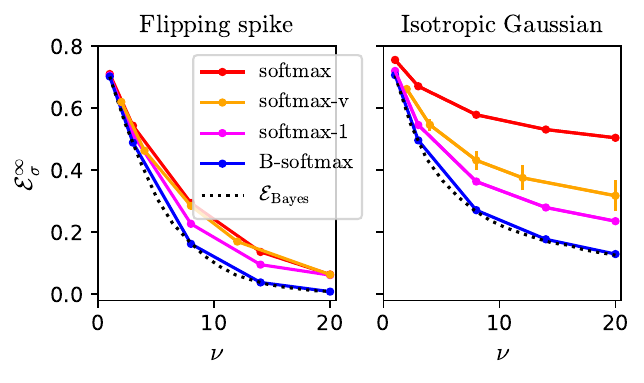}
    \caption{\label{fig:compActivationFunctions_vsHNu} Predicted error $\mathcal{E}^\infty_\sigma$ of the different activation functions. $L=5$, $\eta=1$. Left: $F=4$, $\nu=10$, varying $H$; right: $F=2$, $H=4$, varying signal strength. We performed 5 independent runs.}
\end{figure}

\begin{proposition}[Expressivity of softmax, softmax-v and softmax-1]
\label{res:exprSoftmax}
Assume that there is some disjoint $S\subset\mathbb R^F$ and $\bar S=\{-\theta, \theta\in S\}$ such that $P_\theta(S)$ and $P_\theta(\bar S)$ are bounded away from zero by some constant. Then the softmax and softmax-v attentions are not well specified, i.e. $\mathcal E_\mathrm{softmax}(k,0,0)$ and $\mathcal E_\textnormal{softmax-v}(k,0,0)$ are bounded away from zero for all $k$.

Consider the attention equipped with the softmax-1 activation function. Assume that $||\theta||_2>B$ almost surely. Then in limit of large signal $B\to\infty$ (taken after $D\to\infty$) the softmax-1 attention is well specified, i.e. there is $k$, $b$ and $v$ such that $\mathcal E_\textnormal{softmax-1}(k, b, v)\to 0$.
\end{proposition}
Prop. \ref{res:exprSoftmax} is illustrated by Fig.~\ref{fig:compActivationFunctions_vsHNu}. For the flipping sign direction, the signal $\hat k$ is restricted to a quadrant of $\mathbb R^F$ and therefore all the heads of the softmax(-v) can be positively aligned with it; the softmax(-v) and the softmax-1 have close performances for all signal strengths. This has to be contrasted with the isotropic Gaussian distribution, where the heads cannot always be aligned with $\hat k$: the gap between softmax(-v) and softmax-1 increases with the signal strength $\nu$, and the softmax plateaus at large $\nu$. 

At the same time, the B-softmax is very close to the Bayes risk and achieve better performances, as expected. Independent of its link with the Bayes estimator, the superiority of the B-softmax over the softmax-1 can be explained by its ability to perform normalization adapting not only to the scores of each head separately but also to the "most confident" heads for each input sequence. In Appendix \ref{secApp:activations_comp} we provide additional experiment on the pruning of the trained heads \ref{fig:ablationHeads} and the attention maps for different activation functions \ref{fig:attentionMaps}. They show that the B-softmax and the softmax-1 focus more on the single relevant token and are more specialized.
~\\

\textbf{Semi-realistic data. } The difference between the activations can also bee observed when training on more realistic data. In Fig.~\ref{fig:mnist_vsH} (left) we compare the key-vectors of different heads trained on the MNIST detection task (described in App.~\ref{secApp:mnist}). We see that, as indicated by our discussion in section \ref{sec:dynamics}, softmax attention learns representations where digits are mixed together in every head, while B-softmax produces clearly distinct digit patterns in different heads. Moreover, as can be seen in Fig. \ref{fig:mnist_vsH} (right), the final errors of different activation functions behave as expected from our previous analysis, with softmax-1 and B-softmax outperforming standard softmax.

\begin{figure}[h]
\centering
\includegraphics[width=0.4\linewidth]{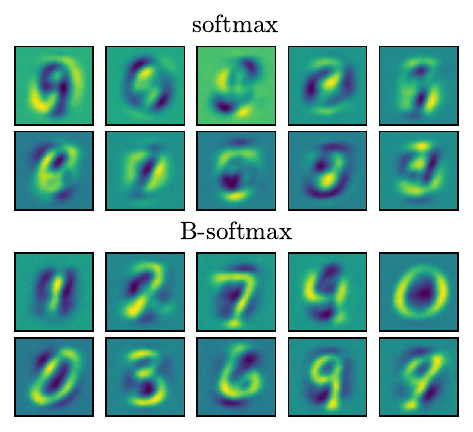}
~~
\includegraphics[width=0.45\linewidth]{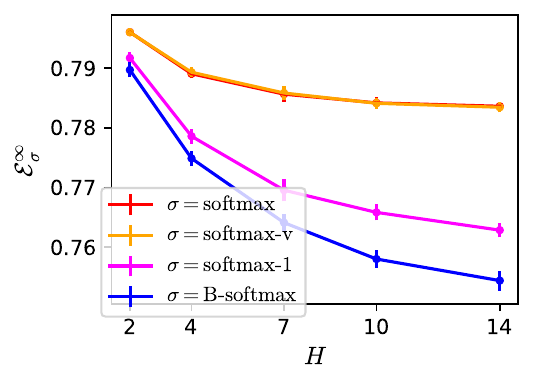}
\caption{\label{fig:mnist_vsH} MNIST detection task (described in App.~\ref{secApp:mnist}). Left: learned $(k_h)_{h\in[H]}$ at $H=10$; right: achieved error $\mathcal{E}^\infty_\sigma$; for the different activation functions after training.}
\end{figure}

\section{Conclusion and limitations.}
We introduced a solvable high-dimensional model of multi-head attention where redundant heads induce persistent variance unless explicitly suppressed by the activation. By restricting to this minimal architecture, we obtained an exact description of training dynamics and isolated staged head specialization, head redundancy, and the role of attention normalization under SGD training.

Our model is intentionally simplified. It considers a single attention layer, excludes output projections and residual pathways, and focuses on a stylized sequence-to-token regression task. Extending the analysis to deeper architectures and more structured data distributions remains an open direction. As in the theory of the multi-index model \cite{aubin2018committee}, the population-level training dynamics derived here are expected to translate into sharp sample-complexity transitions under empirical risk minimization, suggesting a route toward understanding data-dependent phase transitions in attention-based models.

Finally, comparing the overall phenomenology discovered in our model with recent works on head specialization in in-context learning \cite{saxe25iclLinear,chen24icl}, we see that stage-wise head emergence appears in both settings. By contrast, the effect of head redundancy is model-dependent: in our setting, redundant heads induce persistent variance unless suppressed by the attention normalization, while in ICL regression, redundant components are asymptotically harmless.

\section*{Acknowledgment}
We thank Pierre Marion and Claire Boyer for discussion about the single location models and Ludovic Stephan for discussion about the correspondence between SGD and GF dynamics.

We acknowledge funding from the Swiss National Science Foundation grants SNSF SMArtNet (grant number 212049), and the Simons Collaboration on the Physics of Learning and Neural Computation via the Simons Foundation grant (\#1257413 (LZ)).

\bibliography{main}


\appendix

\allowdisplaybreaks
\input{appendix.tex}

\clearpage
\newpage

\section{Supplementary figures}

\subsection{Characterization of the training dynamics}
\label{sec:training_dynamics_plots}
We compare our theoretical prediction, Prop.~\ref{res:sgd}, to the numerical simulation of SGD at finite large $D$. We consider the different activation functions, on the flipping spike or non-isotropic Gaussian.

\begin{figure}[h!]
    \centering
    \includegraphics[width=\linewidth]{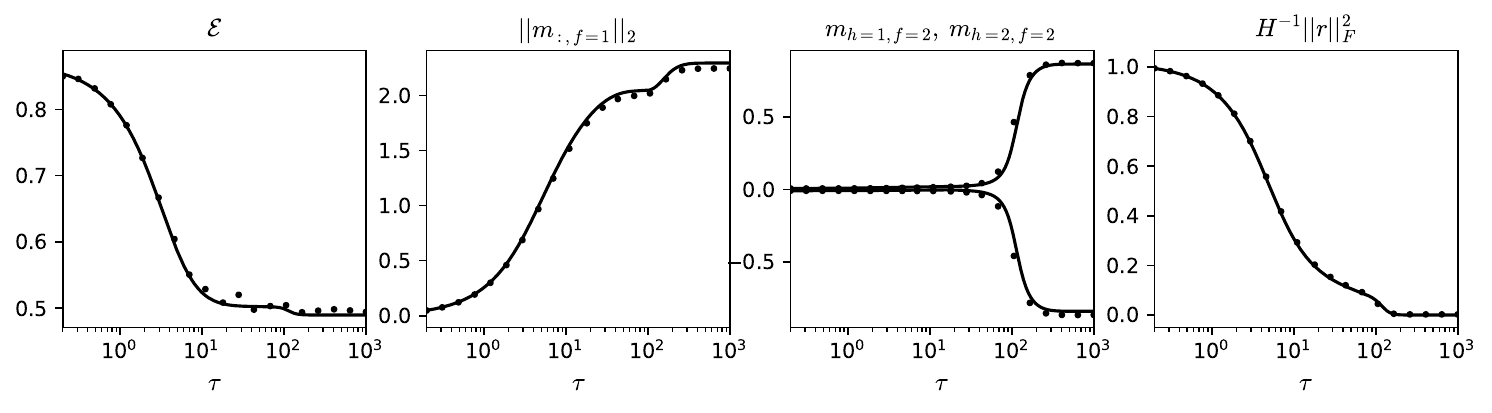}
    \includegraphics[width=\linewidth]{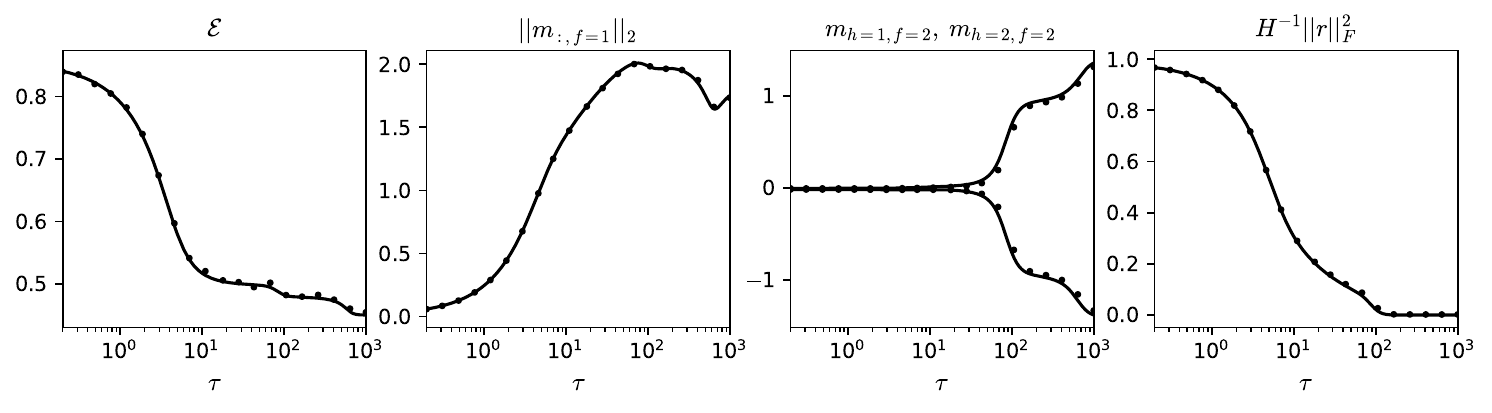}
    \includegraphics[width=\linewidth]{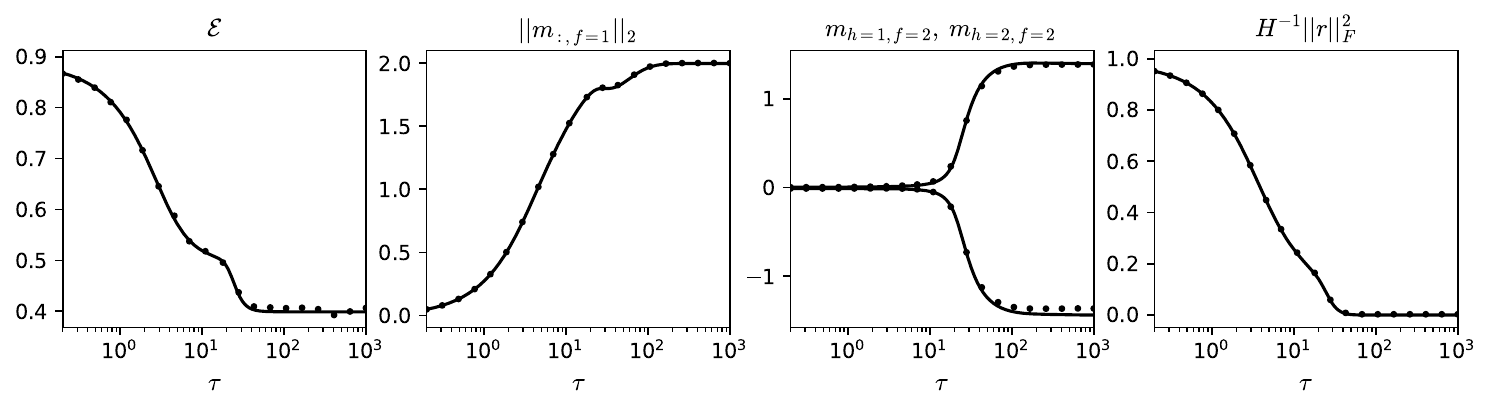}
    \caption{\label{fig:compDfiniVsFlot_supp} Asymptotic description of the attention trained by SGD. We compare numerical simulations at finite $D=10^4$ (dots) and the theoretical description stated in proposition~\ref{res:sgd} (continuous lines). We consider sequence length $L=5$, $H=2$ heads, $F=2$ features and $\theta$ distributed according to the flipping spike distribution, with signal strengths $\nu_1=\nu_2=2$. Initialization $\eta=1$. Top: softmax; middle: softmax-1; bottom: B-softmax.}
\end{figure}

\newpage
\begin{figure}[ht!]
    \centering
    \includegraphics[width=\linewidth]{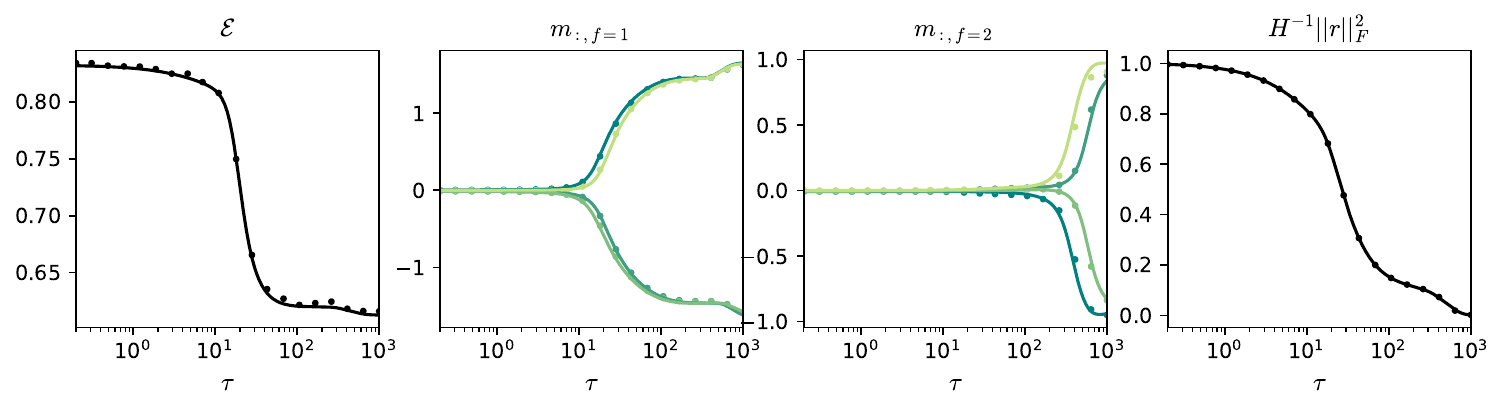}
    \includegraphics[width=\linewidth]{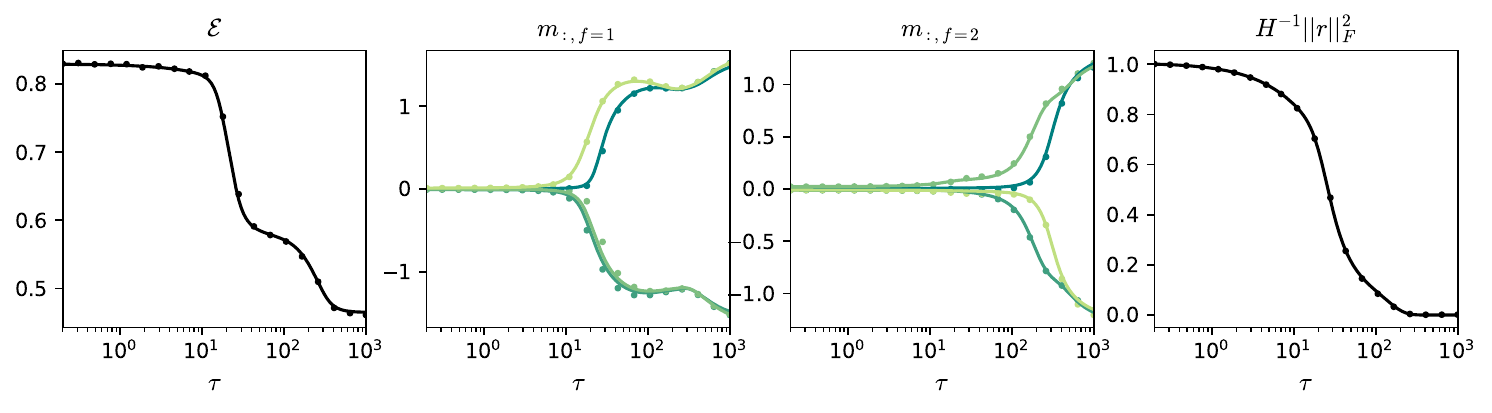}
    \includegraphics[width=\linewidth]{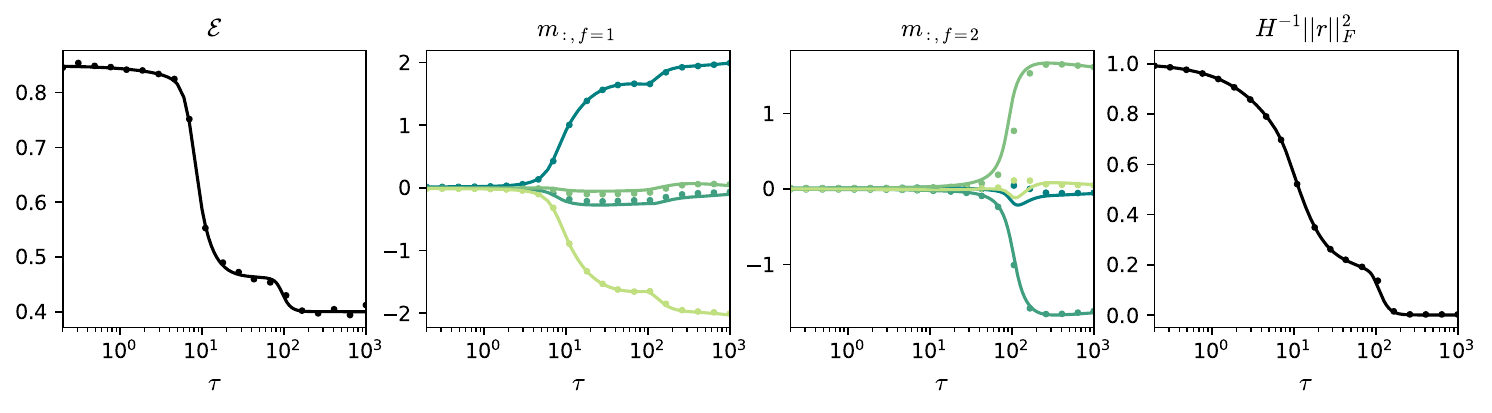}
    \caption{\label{fig:compDfiniVsFlot_suppsuppsupp} Asymptotic description of the attention trained by SGD. We compare numerical simulations at finite $D=10^4$ (dots) and the theoretical description stated in proposition~\ref{res:sgd} (continuous lines). We consider sequence length $L=5$, $H=4$ heads, $F=2$ features and $\theta$ distributed according to the non-isotropic Gaussian distribution, with signal strengths $\nu_1=8, \nu_2=2$. Initialization $\eta=1$. Top: softmax; middle: softmax-1; bottom: B-softmax.}
\end{figure} 

\clearpage
\newpage

\subsection{Specialization of the heads}
\label{sec:specialization_vs_nu}
We provide an additional figure illustrating the specialization of the heads for the flipping spike distribution, depending on the signal strength $\nu_1$ of the average direction and the signal strength $\nu_2$ of the flipping-sign direction. The specialization of the two heads is measured as their cosine similarity and is reported in Fig.~\ref{fig:cosin_by_SNR} for the softmax. It shows that the specialization grows monotonically with $\nu_1$ and decreases monotonically with $\nu_2$, as expected.

\begin{figure}[h!]
    \centering
    \includegraphics[width=0.75\linewidth]{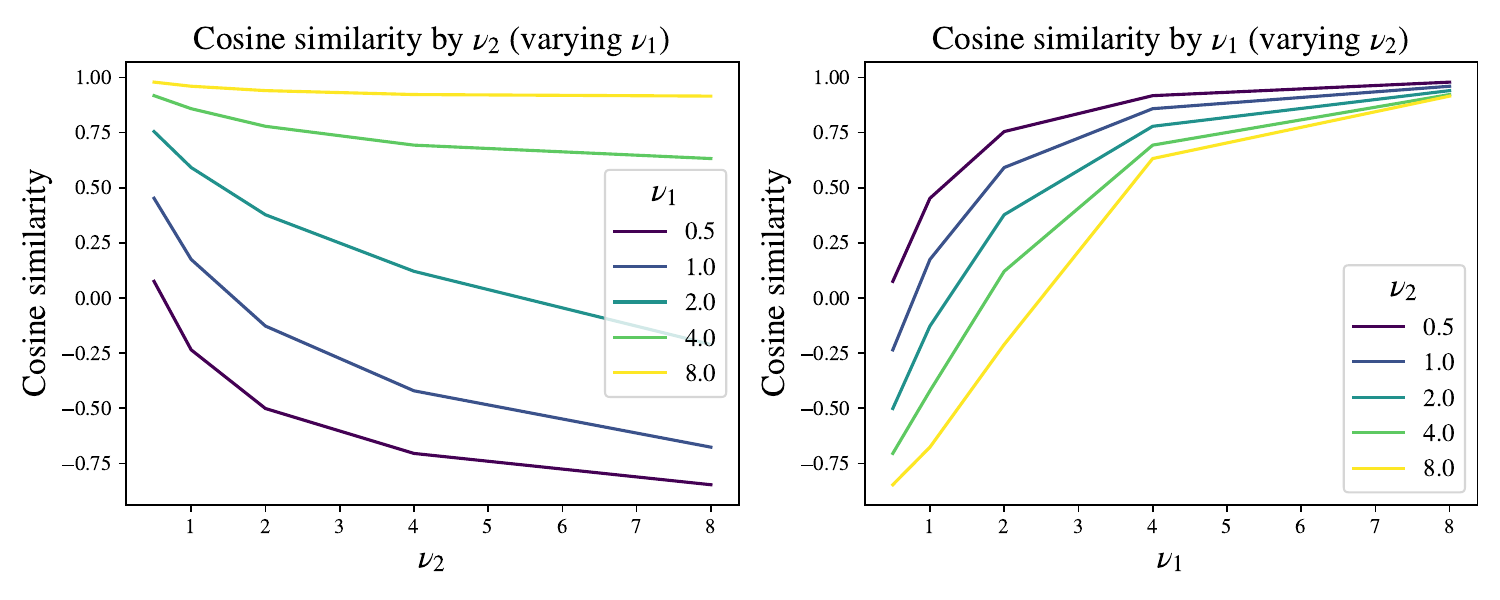}
    \caption{Cosine similarity between heads for softmax attention with $H=2$, and flipping spike distribution with $F=2$, $L=4$ depending on the signal strengths for the constant and flipping-sign directions. }
    \label{fig:cosin_by_SNR}
\end{figure}

Figure \ref{fig:specialization_vs_H_4_panels} shows that when increasing the number of heads $H$, some heads specialize and diverge more from the average direction, more than at lower $H$.
\begin{figure}[h!]
    \centering
    \includegraphics[width=\linewidth]{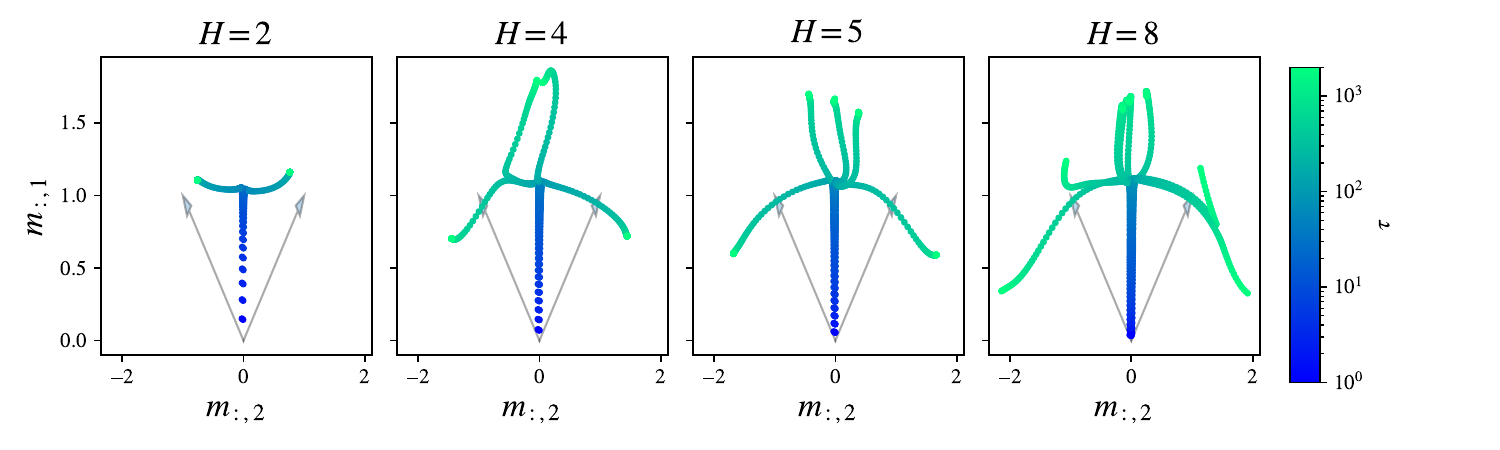}
    \caption{Evolution of the heads in the model of attention with varying number of heads $H\in \{2, 4, 5, 8\}$, and flipping spike distribution with $F=2$, $L=4$, with signal strengths $\nu_1=\nu_2=2$. $\eta=1$.}
    \label{fig:specialization_vs_H_4_panels}
\end{figure}

\clearpage
\newpage

\subsection{Sequential specialization}

We provide Figs.~\ref{fig:ellipsoide_softmax} and \ref{fig:hierarchical3D_bis} that complement the description of the specialization dynamics of section \ref{sec:dynamics}, showing the specialization of the heads for an anisotropic Gaussian distribution for different $H$ or the different activation functions. The three different runs in Fig.~\ref{fig:hierarchical3D_bis} show that the behaviour of the heads is similar across initializations.

\begin{figure}[h!]
    \centering
    \includegraphics[width=0.75\linewidth]{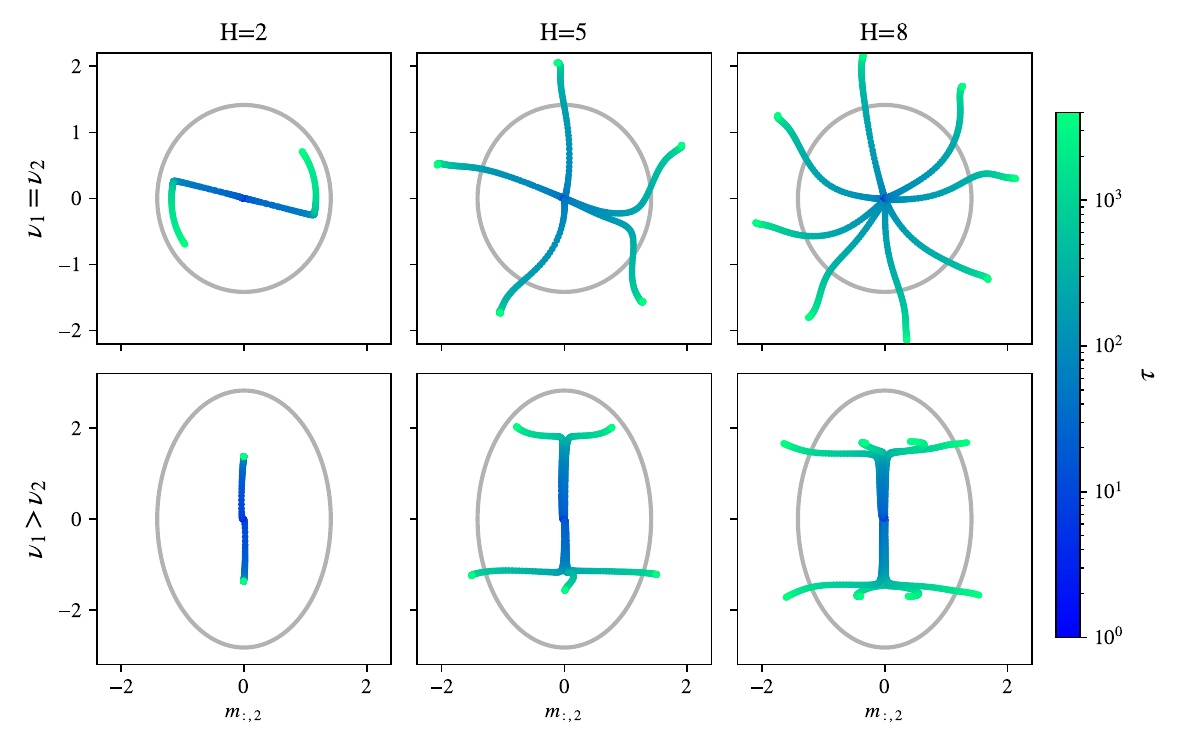}
    \caption{Evolution of the heads attention with varying number of heads $H\in \{2, 5, 8\}$, and non-isotropic Gaussian distribution with $F=2$, $L=4$, with signal strengths $\nu_1=\nu_2=2$ (top) and $\nu_1=8, \nu_2=2$ (bottom). $\eta=1$.}
    \label{fig:ellipsoide_softmax}
\end{figure}

\begin{figure}[h!]
    \centering
    \includegraphics[width=0.92\linewidth]{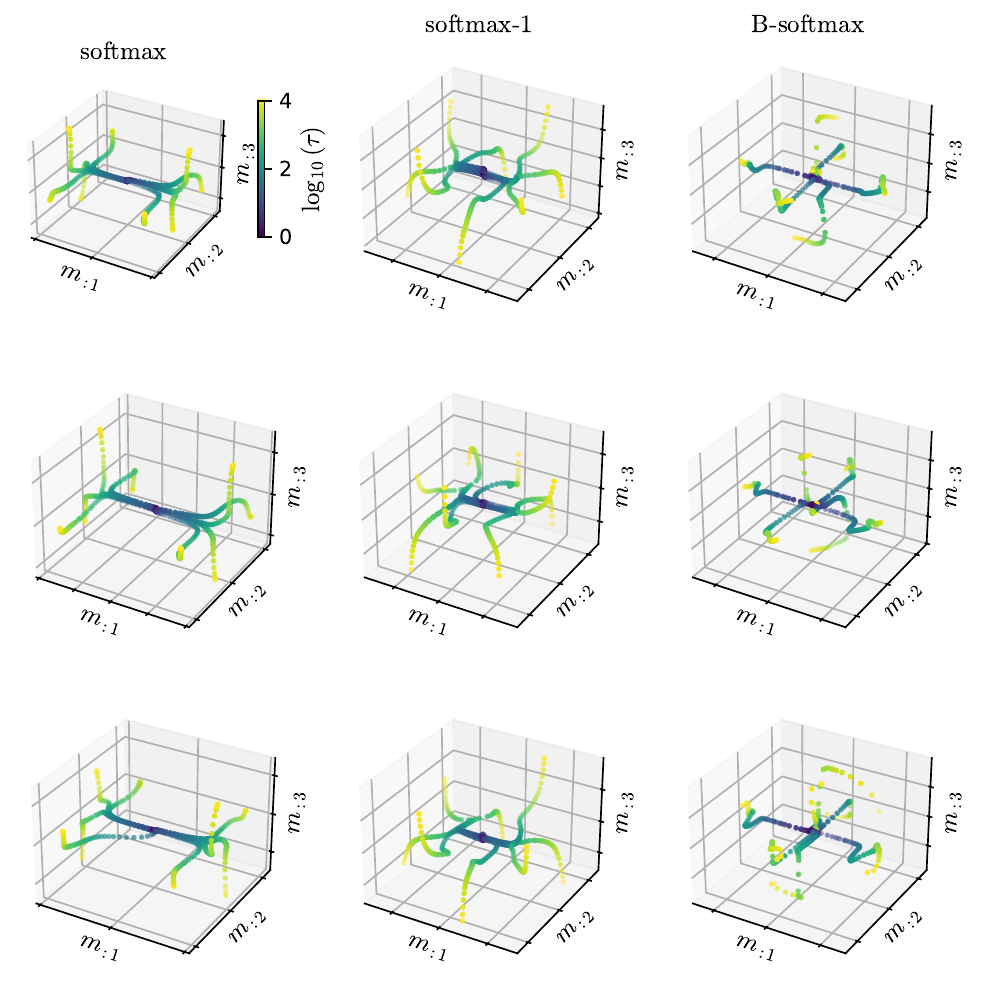}
    \caption{\label{fig:hierarchical3D_bis} Evolution of the $H=8$ heads for the non-isotropic Gaussian distribution at $F=3$. The 2 lines are 2 different runs with different initial conditions.
    }
\end{figure}

\clearpage
\newpage

\subsection{Comparison of the different activations after training}
\label{secApp:activations_comp}
In this section we provide further elements of comparison of the different activations after training, complementing the discussion of part~\ref{sec:activations}.

The attention maps of the different attentions are shown in Fig.~\ref{fig:attentionMaps} for a few sequences drawn from the isotropic Gaussian distribution. The attention map of the softmax is noisy: the heads that do not focus on the relevant token have to focus on irrelevant ones, while the softmax-1 reduces attention scores of the ``irrelevant'' heads, and this effect is even more prominent for the B-softmax.

The softmax-1 and B-softmax rely on the multiple heads in a more optimal and specialized way. To see this, we perform a head pruning experiment and compare the importance of the heads. The pruning of the trained heads is done in a greedy manner and we uniformly rescale the output to keep constant the amplitude of the attention scores. The results of the experiment are reported on Fig.~\ref{fig:ablationHeads}. Similarly to previous empirical studies \cite{michel2019sixteenheadsreallybetter, voita2019multiheadspecialized}, we observe that a substantial number of heads can be removed without seriously affecting performance. In our model, in the case of a flipping spike distribution over $F$ spikes or an isotropic Gaussian in dimension $F$, the number of heads $\tilde H$ that can be pruned without significant loss of performance is close to $H - F$; i.e. we can keep one head per each hidden feature approximately. If one prunes more heads $\tilde H>H-F$ and removes the heads that are actually necessary for the inference, the softmax-1 and the B-softmax behave differently from the standard softmax. Their performances degrade more severely than softmax and with a larger variance over the repeated runs. This suggests that attention with softmax-1 or B-softmax activation strongly relies on all the necessary heads together, and that these are strongly specialized.

\begin{figure}[ht]
  \centering
  \begin{minipage}{0.485\textwidth}
    \centering
    \includegraphics[width=0.95\linewidth]{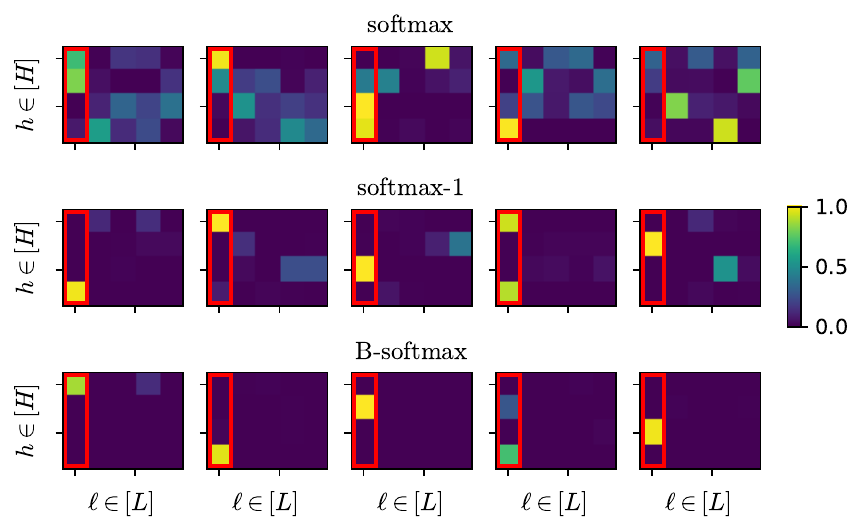}
    \caption{\label{fig:attentionMaps} Attention maps of the different activation functions after training. $H=4$ heads, $L=5$ and the relevant token is highlighted in the red rectangle. We show the attention maps for five independent sequences. $F=3$, signal isotropic Gaussian of strength $\nu=9$.}
  \end{minipage}
  \hfill
  \begin{minipage}{0.485\textwidth}
    \centering
    \includegraphics[width=0.99\linewidth]{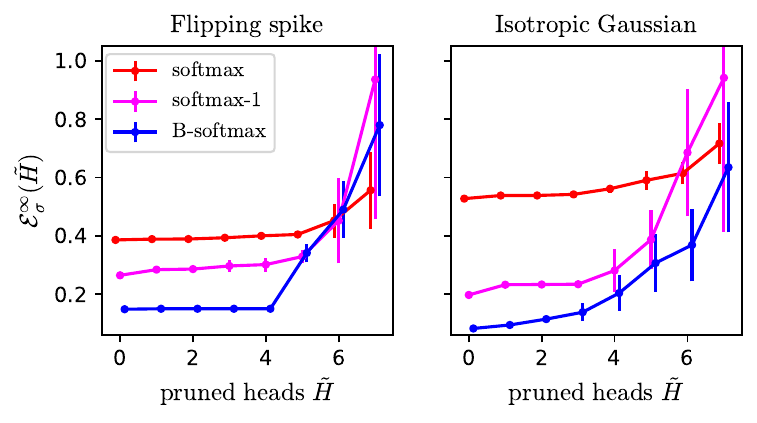}
    \caption{\label{fig:ablationHeads} Pruning head experiment. Error $\mathcal{E}^\infty_\sigma(\tilde H)$ of the different activation functions after training, after pruning $\tilde H$ heads over $H=8$ total heads. $L=5$, $F=4$, $\nu=10$. We performed five training runs with different initial conditions.}
  \end{minipage}
\end{figure}

\end{document}

%% file: appendix.tex
\section{Proofs of propositions and lemma.}
In this part we give the different proofs and justifications for our theoretical results of parts \ref{sec:dynamics} and \ref{sec:activations}.

\subsection{Low-dimensional characterization of SGD}
\label{secApp:low_dim_sgd}
We start by proving the asymptotic characterization of the training dynamics in terms of the order parameters.

\begin{proof}[Proof of Proposition \ref{res:paramLoss}]
We first derive the order parameters and the expression for the reparameterized loss.
Notice that in the high-dimensional limit when $D\to\infty$, the hidden spikes directions $k^{*}_f$ are almost orthogonal. Then parameters of the attention model $k_{1},\ldots, k_{H}$ can be expressed as follows:
\begin{equation}
    k_{h} = \sum_{f=1}^F m_{hf}k^{*}_f + \sum_{h'=1}^H r_{hh'}k^{\perp}_{h'}\ , 
\end{equation}
where $\{k^{\perp}_{h'}\}_{h'\in[H]}$ are vectors orthogonal to $\{k^{*}_f\}_{f\in[F]}$ and between each other, and
\begin{align}
& m_{hf} =(k_h)^\top k_f^*\ ,\quad h\in [H], \; f\in [F] && q_{hh'} =(k_h)^\top k_{h'}\ ,\quad h,h'\in [H] \\
& p_{ff'} = (k_f^*)^\top k_{f'}^*\approx \delta_{f,f'}\ ,\quad f,f'\in [F] && r = (q-mp^{-1}m^\top)^{\sfrac{1}{2}}
\end{align}

Thus, the pre-activations $\chi_{h}$ of the attention head $h$ can be expressed as
\begin{equation}
    \chi_{h} = \sum_{f=1}^Fm_{hf}Xk^{*}_f + \sum_{h'=1}^Hr_{hh'}Xk^{\perp}_{h'},
\end{equation}
or, if we denote $\chi^{*}_f=Xk^{*}_f$ and $\xi_{h'} = Xk^{\perp}_{h'}$
\begin{equation} \label{6_eq:preactivation}
    \chi_{h} = \sum_{f=1}^F m_{hf}\chi^{*}_f + \sum_{h'=1}^H r_{hh'}\xi_{h'}.
\end{equation}

The risk depends on the parameters only via the projections $\chi_{h}$, and in the high-dimensional limit the projections $\chi^{*}_f$ and $\xi_{h'}$ are normally distributed, such that all elements of the vectors are independent, $\chi^{*}_{f\ell}\sim\mathcal{N}(\delta_{\ell, \epsilon^*}\theta_f, 1)$, and $\xi_{h'\ell}\sim\mathcal{N}(0, 1)$.

We can now express the loss in terms of the order parameters. We begin by explicitly writing the output of the attention model and expanding the norm squared:
\begin{align}
  &\mathcal{E}_\sigma(k, b, v) = \frac{1}{D}\mathbb E_{\epsilon,\theta,X}\left[\|X_{\epsilon}\|^2_2 - \frac{2}{H}\sum_{h=1}^H \sigma(\chi, b, v; h)^TXX_{\epsilon}\right. \\
& \qquad\left.{}+ \frac{1}{H^2}\sum_{h=1}^H \sum_{h'=1}^H\sigma(\chi, b, v; h)^TXX^T\sigma(\chi, b, v; h')\right]. \nonumber
\end{align}
This can be simplified using independence of the tokens $X_\ell$ and noticing that $\mathbb E_{X}\|X_\ell\|^2_2 \to D$ when $D\to\infty$.
\begin{align}
  &\mathcal{E}_\sigma(k, b, v) = \frac{1}{D}\mathbb E_{\epsilon,\theta,X}\left[\|X_{\epsilon}\|^2_2 - \frac{2}{H}\sum_{h=1}^H \sigma(\chi, b, v; h)_\epsilon\|X_{\epsilon}\|^2_2 \right. \\
  & \qquad\left.{}+ \frac{1}{H^2}\sum_{h=1}^H \sum_{h'=1}^H\sum_{\ell=1}^L\sigma(\chi, b, v; h)_\ell\sigma(\chi, b, v; h')_\ell\|X_\ell\|^2_2\right] \nonumber \\
    & = \mathbb E_{\epsilon,\theta,X}\sum_{\ell=1}^L \left(\delta_{\ell, \epsilon} - \frac{1}{H}\sum_{h=1}^H \sigma(\chi, b, v; h)_\ell\right)^2.
\end{align}

Using the expression \eqref{6_eq:preactivation}, we get the final expression for $\tilde{\mathcal{E}}$.
\end{proof}

\subsection{Description of the dynamics}
\label{secApp:phases}

\subsubsection{Sample complexity of SGD}
\label{secApp:sampleComplexitySGD}
In this section we formally prove the convergence of the trajectories of the order parameters under SGD to the GF in the limit and derive sample complexity of weak recovery with SGD using Theorem 2.3 from \cite{arous2023highdimensionallimittheoremssgd}.

To do so, we should prove that the triple: summary statistics (order parameters) $\mathbf{u}^D=(\{m_{hf}\}_{h\in[H], f\in[F]}, \{r_{hh'}\}_{h, h'\in[H]},$ $\{b_h\}_{h\in[H]}, v)$, loss function \eqref{eq:empirical_loss} and data distribution $P_D$(as defined in section \ref{sec:data_distribution}) is $\gamma$-localizable according to Definition 2.1 in \cite{arous2023highdimensionallimittheoremssgd}. And find $\mathbf{h}:\mathbb{R}^{k}\to \mathbb{R}^{k}$ and $\mathbf{\Sigma}:\mathbb{R}^{k}\to \mathbb{R}^{k\times k}$ (where $k=HF + HH + H + 1$) such that the family of summary statistics is asymptotically closable with effective drift $\mathbf{h}$, and effective volatility $\mathbf{\Sigma}$.

Following \cite{arous2023highdimensionallimittheoremssgd}, we denote $$H(\{X^\mu, y^\mu\}_{\mu=1}^{N_b}; k, b, v) = \mathcal{L}(\{X^\mu, y^\mu\}_{\mu=1}^{N_b}; k, b, v) - \mathcal{E}(k, b, v).$$ We omit the activation function $\sigma$ to lighten the notation. We also denote $J_{D}=\nabla \mathbf{u}^{D}$, $V(x) = \mathbb{E}[\nabla H(x) \otimes \nabla H(x)]$.

\begin{lemma}[Moments of loss deviation partial derivatives]
\label{resApp:LossDeviationMoments}
    Let $\theta$ be some parameter $\theta \in \{k_{hd}\}_{h\in[H], d\in[D]}\cup\{b_h\}_{h\in[H]}\cup\{v\}$. Denote 
    $$Y_{\theta} = \frac{1}{D}\partial_{\theta}\|y - \hat{y}_{\sigma, k, b, v}(X)\|^2,$$
    where $\hat{y}$ is given by \ref{eq:attention_model}, for some fixed values of $k, b, v$. This random variables have finite absolute moments:
    $$\mathbb{E}[|Y_{\theta}|^{q}] = O(1), \quad q\in\mathbb{N}.$$
    Moreover, even moments of partial derivatives of $H$ are bounded by:
    $$\mathbb{E}[(\partial_{\theta} H)^{2q}] = O(1/N_b^q).$$
\end{lemma}
\begin{proof}
    The expressions for partial derivatives of the loss are
    \begin{align}
        \partial_{k_{hd'}} \frac{1}{D}\partial_{\theta}\|y - \hat{y}_{\sigma, k, b, v}(X)\|^2 = &\frac{2}{D}\sum_{d=1}^D \sum_{\ell=1}^LX_{\ell d}\left(\frac{1}{H}\sum_h\sigma_{\ell}(X; h) - \delta_{\ell\epsilon}\right)\\
        &\sum_{\ell'=1}^LX_{\ell' d}\left(\frac{1}{H}\sum_{h'}\partial_{k_{hd'}}\sigma_{\ell'}(X; h')\right) \\
        = &\sum_{l, l', l''}^L c_{h}(X^\mu; l, l', l'') X_{d'l''}\left(\frac{1}{D}\sum_{d}^DX_{dl}X_{dl'}\right),\\
        \partial_{b_{h}} \frac{1}{D}\partial_{\theta}\|y - \hat{y}_{\sigma, k, b, v}(X)\|^2 
        &= \sum_{l, l'}^L c_{h}(X; l, l') X_{d'l'}\frac{1}{D}\left(\sum_{d}^D X_{dl}\right),
    \end{align}
    where $c_{h}(X; l, l'), c_{h, d'}(X; l, l', l'')$ are bounded by some constant $c_\sigma$ due to the activation $\sigma$ being bounded for fixed values of $v$ and $b_h$.
    
    Notice, that random variables $Y_\theta$ has the form $\frac{1}{D}\sum_d^D A_d$. We can estimate $\mathbb{E}|\frac{1}{D}\sum_d^D A_d|^q \leq \frac{1}{D^q} \sum_{d_1}^D\ldots \sum_{d_q}^D \mathbb{E}\left[\prod_{i=1}^q|A_{d_i}|\right]$.
    In our case $|A_{d_i}|\leq c_\sigma|X_{ld}X_{l'd}X_{l''d'}|$ or $|A_{d_i}|\leq c_\sigma|X_{ld}X_{l'd'}|$, where $X_{ld}$ are normally distributed. In both cases, $\mathbb{E}\left[\prod_{i=1}^q|A_{d_i}|\right] \leq C_\sigma \mathbb{E}\left[\prod_i|X_{l_id_i}X_{l'_id_i}X_{l''_id'_i}|\right](\text{ or }\mathbb{E}\left[\prod_i|X_{l_id_i}X_{l'_id'_i}|\right]) = O(1)$. 
    Summing all this expectations, we get $\mathbb{E}|Y_\theta|^q = O(1)$.

    Notice that $\partial_{\theta} H $ is an average of $N_b$ independent centered random variables $Y^\mu_{\theta}-\mathbb{E} Y^\mu_{\theta}$ for $\mu\in[N_b]$. Then
    \begin{align}
        \mathbb{E}[(\partial_{\theta} H)^{2q}] &= \frac{1}{N_b^{2q}} \sum_{\mu_1}^{N_b}\ldots\sum_{\mu_{2q}}^{N_b}\mathbb{E}\left[\prod_i^{2q}Y^{\mu_i}\right] \\
        &= \frac{1}{N_b^{2q}} \sum_{\mu_1}^{N_b}\ldots\sum_{\mu_{k}}^{N_b}\mathbb{E}\left[\prod_i^{q}(Y^{\mu_i})^2\right] \\
        &= \frac{1}{N_b^{2q}} N_b^q O(1) = O(1/N_b^q),
    \end{align}
    where the second transition is due to $Y^{\mu_i}$ and $Y^{\mu_j}$ being independent when $\mu_i \neq \mu_j$, and centered. So, if the indices $\mu_1, \ldots, \mu_{2q}$ are not ``paired'', the term is zero in expectation.
\end{proof}

\begin{lemma}[Summary statistics are $\gamma$-localizable]
    The family of summary statistics $\mathbf{u}^D$ 
    \begin{align}
    \label{eqApp:sufficientStatistics}
        &\mathbf{u}^D = (\{m_{hf}\}_{h\in[H], f\in[F]}, \{r_{hh'}\}_{h, h'\in[H]}, \{b_h\}_{h\in[H]}, v),\\
        &\mathrm{where}\quad m_{hf} = \langle k_h, k^*_f\rangle, \quad r_{hh'} = \sqrt{\langle k_h, k_{h'}\rangle - \sum_f m_{hf}m_{h'f}}
    \end{align}
    is such that the triple ($\mathbf{u}^D$, $\mathcal{L}$, $P_D$) is $\gamma$-localizable with sequence of compacts $E_K = \bar{B}_K(0)\setminus B_{1/K}(0)\in\mathbb{R}^{k}$.
\end{lemma}
\begin{proof}
\textbf{Property (1).} Second derivative of the summary statistics in not zero only for $r_{hh'}$:
\begin{align}
        \nabla_{k_h} r_{hh'} &= \frac{1}{2r_{hh'}}(k_{h'} - \sum_f m_{h'f}k^*_f) \\
        \nabla^2_{k_h, k_h} r_{hh'} &= -\frac{1}{2r^2_{hh'}}(k_{h'} - \sum_f m_{h'f}k^*_f)(k_{h'} - \sum_f m_{h'f}k^*_f)^T = -\frac{1}{2\scprod{k^{\perp}_{h}}{k^{\perp}_{h'}}^2} k^{\perp}_{h'}(k^{\perp}_{h'})^T\\
        \nabla^2_{k_h, k_{h'}} r_{hh'} &= \frac{1}{2\scprod{k^{\perp}_{h}}{k^{\perp}_{h'}}}(I_D - \sum_f k^*_f(k^*_f)^T) -\frac{1}{2\scprod{k^{\perp}_{h}}{k^{\perp}_{h'}}^2} k^{\perp}_{h'}(k^{\perp}_{h})^T\\
        \nabla^3_{k_h} r_{hh'} &= \frac{1}{\scprod{k^{\perp}_{h}}{k^{\perp}_{h'}}^3} (k^{\perp}_{h'})^{\otimes 3} \\
        \nabla^3_{k_h, k_h, k_{h'}} r_{hh'} &= -\frac{1}{2\scprod{k^{\perp}_{h}}{k^{\perp}_{h'}}^2_{hh'}}(I_D - \sum_f k^*_f(k^*_f)^T)\otimes k^{\perp}_{h} +\frac{1}{\scprod{k^{\perp}_{h}}{k^{\perp}_{h'}}^3} k^{\perp}_{h'} \otimes k^{\perp}_{h'} \otimes k^{\perp}_{h},
    \end{align}
where we denote $k^\perp_h = k_{h} - \sum_f m_{hf}k^*_f$
We get that the operator norm is bounded by some constant $C_K$ as soon as $\|k^{\perp}_h\|$ is bounded, then it is bounded on $E_K = \bar{B}_K(0)\setminus B_{1/K}(0)$.

\textbf{Property (2).} First, notice that 
$$\|\nabla\Phi\|^2 \leq \sum_{u\in\mathbf{u}_D}|\partial_{u}\Phi|\|\nabla u\|^2.$$
Here, $\|\nabla u\|^2$ is bounded by constant on the compact. And $\Phi(\mathbf{u}_D)$ is such that $$\partial_u\Phi = \mathbb{E}_{\epsilon, \theta, \chi_D, \xi_D}\sum_\ell\phi_\ell(\epsilon, \theta, \chi_D, \xi_D; \mathbf{u}_D)(\delta_{u = m_{hf}}(\chi_D)_\ell + \delta_{u = r_{hf}}(\xi_D)_\ell + \delta_{u = b_h, v}),$$ where $|\phi_\ell(\epsilon, \theta, \chi_D, \xi_D; \mathbf{u}_D)| = O(1)$ is bounded and random variables $\chi_D, \xi_D$ weakly converge to normal distribution $\chi_D, \xi_D \to \mathcal{N}(0, 1)$, therefore $|\partial_u \Phi(\mathbf{u}_D)| = O(1)$.

Now, to bound $\mathbb{E}[\|\nabla H\|^8]$, by applying Cauchy–Schwartz inequality twice, we get
\begin{align}
    \|\nabla H\|^8 = \left(\sum_{\theta}(\partial_\theta H)^2\right)^4\leq p_D^3 \sum_{\theta=1}^{p_D}(\partial_\theta L_D)^8 
\end{align}
and using \ref{resApp:LossDeviationMoments}, we have 
\begin{align}
    \mathbb{E}[(\partial_\theta H)^8] = O\left(\frac{1}{N_b^4}\right),
\end{align}
Therefore we get $\mathbb{E}[\gamma^4\|\nabla L_D\|^8 ]=O(1)$ and the required bound is satisfied.

\textbf{Property (3).} 
Notice that $\scprod{\nabla H}{\nabla (b_h/v)}^4 = (\partial_{b_h/v}H)^4 = O(1)$, while 
\begin{align}
    \scprod{\nabla H}{\nabla (m_{hf}/r_{hh'})}^4 &= \left(\sum_d (k^*_{fd}/k^{\perp}_{h'd}) \partial_{k_{hd}}H\right)^4 \\
    &\leq ((\|k^*\|^2/\|k^{\perp}_{h'd}\|^2)\sum_d(\partial_{k_{hd}}H)^2)^2 \leq C_K D \sum_d (\partial_{k_{hd}}H)^4,
\end{align}
and applying lemma \ref{resApp:LossDeviationMoments}, we get $\gamma^2\mathbb{E}\scprod{\nabla H}{\nabla (b_h/v)}^4 \leq \gamma^2O(D^2N_b^{-2}) = O(1)$.

Second derivative is only non-zero for the statistics $r_{hh'}$ and we get 
\begin{align}
    |\scprod{\nabla^2u}{\nabla H\otimes \nabla H }| &= |\text{Tr}(\nabla^2u (\nabla H\otimes \nabla H))| \\ 
    &= \frac{1}{2r_{hh'}} (\scprod{k^{\perp}_h}{\nabla H}^2 + \scprod{k^{\perp}_{h'}}{\nabla H}^2 + 2\scprod{k^{\perp}_h}{\nabla H}\scprod{k^{\perp}_{h'}}{\nabla H}) \\
    & + \frac{1}{r_{hh'}}|(\|\nabla H\|^2 - \sum_f \scprod{k^*_f}{\nabla H}^2)| \\
    &\leq \frac{2}{r_{hh'}} (\max\{\scprod{k^{\perp}_h}{\nabla H}, \scprod{k^{\perp}_{h'}}{\nabla H}\})^2 + \frac{1}{r_{hh'}}(1 + \sum_f\|k^*_f\|^2)\|\nabla H\|^2,
\end{align}
Notice that both $\mathbb{E} \scprod{\nabla^2u}{\nabla H\otimes \nabla H } = O(DN_b^{-1})$ and $\mathbb{E} \scprod{\nabla^2u}{\nabla H\otimes \nabla H }^2 = O(D^2N_b^{-2})$ (using that $(a+b)^2 \leq 2(a^2 + b^2)$). Thus, we get $\gamma^3 \mathbb{E} \scprod{\nabla^2u}{\nabla H\otimes \nabla H -V }^2 = \gamma O(1) = o(1)$.

\end{proof}

\begin{lemma}[High-dimensional limit of SGD]
\label{resApp:high_dim_limit}
    The family of sufficient statistics \eqref{eqApp:sufficientStatistics} converges (in the sense defined in Theorem 2.3 from \cite{arous2023highdimensionallimittheoremssgd}) to the gradient flow \eqref{eq:dynEff} initialized at $m_{hf}=0, r_{hh'}=1, b_h=0, v=1$ for $h, h'\in[H], f\in[F]$ when $D\to\infty$. 
\end{lemma}
\begin{proof}
Notice that $V$ consist of the entries of order $O(N_b^{-1})$ and $\partial_i\partial_j u(x)$ is non zero only for statistics $r_{hh'}$ but still bounded on the compact $E_K$, therefore $\gamma\mathcal{L}_D\mathbf{u} = \gamma O(D/N_b) = o(1)$.
Notice that either $\partial_\theta u(x)$ is of order $O(D^{-1/2})$ for $O(D)$ parameters $\theta$ or $\partial_\theta u(x)$ is of order $O(1)$ for exactly 1 parameter $\theta$ and for all other parameters $\theta$ $\partial_\theta u(x) = 0$. Thus, $\sum_{i,j}^D V_{ij}J_{ki}J_{lj} = O(D/N_b)$ and $\gamma JVJ^T = o(1)$.

Applying Theorem 2.3 from \cite{arous2023highdimensionallimittheoremssgd}, in the limit we have
$$\dot{\mathbf{u}} = \mathbf{h}(\mathbf{u}),$$
where $\mathbf{h}$ is such that:
$$\sup_{x\in\mathbf{u}_D^{-1}(E_K)} \|\mathbf{h}(\mathbf{u}_D(x)) - \mathcal{A}_D\mathbf{u}_D(x)\|\to 0.$$

Notice that
\begin{align}
    \mathcal{A}_D\mathbf{u}(x) = \sum_\theta \nabla_\mathbf{u}\Phi(\partial_\theta\mathbf{u)}^2
\end{align}
which corresponds to the following system of ODE's:
\begin{align}
    \dot{m}_{hf} &= -(\|k^*_f\|^2)_{D\to\infty}\partial_{m_{hf}}\Phi(\mathbf{u}) = -\partial_{m_{hf}}\Phi(\mathbf{u}) \\
    \dot{r}_{hh'} &= -\left(\frac{\|k^\perp_h\|^2\|k^\perp_{h'}\|^2}{\scprod{k^\perp_{h}}{k^\perp_{h'}}^2}\right)_{D\to\infty}\partial_{r_{hh'}}\Phi(\mathbf{u}) = - \partial_{r_{hh'}}\Phi(\mathbf{u})\\
    \dot{b}_{h} &= -\partial_{b_{h}}\Phi(\mathbf{u}) \\
    \dot{v} &= -\partial_v\Phi(\mathbf{u})
\end{align}
\end{proof}

\paragraph{Unspecialized phase Prop. \ref{res:1stPhase}.}
Assuming $\mathbb E\theta\neq 0$, according to Lemma \ref{resApp:invUnspMan} and \ref{resApp:gradI}, the space $\mathcal{M}_\mathrm{u}$ of unspecialized $m$, $r$ and $b$ is invariant by the gradient flow; and moreover the subspace of $m$ proportional to $\mathds{1}_H(\mathbb E\theta)^\top$ is also invariant by the gradient flow. At initialization $m\approx0$, $r\approx I_H$, $b=0$ are unspecialized; and the gradient of the loss w.r.t. $m$ does not vanish and points towards $\mathds{1}_H(\mathbb E\theta)^\top$. Consequently only a few time steps $\tau^\mathrm{u}$ are necessary to move in this direction, i.e. a time $\tau^\mathrm{u}=\Theta(1)$ is enough so that $m=x\mathds{1}_H(\mathbb E\theta)^\top$ with $x=\Theta(1)\in\mathbb R^+$. The total number of required samples is $N=tN_b=\tau^\mathrm{u}\gamma^{-1}N_b=\omega(D)$.

As to $b$ and $v$, for the softmax-1, according to Lemma \ref{resApp:gradBVI}, the gradient at initialization does not vanish and points towards unspecialized fixed-points that are reached in time $\tau=\Theta(1)$. For the B-softmax the gradient w.r.t. $b$ at initialization is null, $b$ does not move and stays unspecialized.

\paragraph{Escaping the invariant subspace.} According to lemma \ref{resApp:gradI} if $\mathbb{E}\theta=0$, then gradient flow initialized at $m_{h} = 0$ (which is the case, when we initialize the model with $k_h\sim\mathcal{N}(0, \eta I_D)$ with $\eta=O(1)$) remains stuck at this magnetization. Therefore, to obtain meaningful analysis of specialization sample complexity, we should, as suggested by \cite{arous2023highdimensionallimittheoremssgd}, rescale the sufficient statistics to obtain the SDE limiting trajectory of SGD. We believe that similar reasoning can be applied to the recentered sufficient statistics $m^{(D)}_{hf}-m_*$ the case of $\mathbb{E}\theta\neq 0$ when magnetization reaches the unspecialized fixed point $m_*$ in the first phase of the dynamics as suggested by \cite{arous2023highdimensionallimittheoremssgd}.

\begin{lemma}[SDE limit with rescaled magnetizations]
\label{resApp:sgdLimit}
    The family of sufficient statistics with rescaled $m_{hf} \rightarrow \tilde{m}_{hf} = D^{\zeta}m_{hf}$ of the dynamics of SDE with step size $\gamma N_b^{-1} = c_\text{lr}D^{-1-2\zeta}$ for some constant $c_\text{lr}$ and arbitrarily small $\zeta$ converges to the following system of equations:
    \begin{align}
        d\tilde{m}_{hf} &= -\partial_{m_{hf}}\Phi(\mathbf{u})dt + c_{\Sigma} dB_t \\
        \dot{r}_{hh'} &= - \partial_{r_{hh'}}\Phi(\mathbf{u})\\
        \dot{b}_{h} &= -\partial_{b_{h}}\Phi(\mathbf{u}) \\
        \dot{v} &= -\partial_v\Phi(\mathbf{u}),
    \end{align}
    where $c_{\Sigma}>0$ is some constant and $B_t$ is a standard Brownian motion in $\mathbb{R}^{HF}$.
\end{lemma}
\begin{proof}
Let $\gamma N_b^{-1} = c_\text{lr}D^{-1-2\zeta}$ for some constant $c_\text{lr}$, we rescale $m_{hf} \rightarrow \tilde{m}_{hf} = D^{\zeta}m_{hf}$, so that the sufficient statistics are still $\gamma$-localizable. 

After rescaling $m_{hf} \rightarrow \tilde{u} = D^{\zeta}m_{hf}$, the only property that might break is Property (3), which we avoid by setting learning rate $\gamma N_b = c_\text{lr}D^{-1-2\zeta}$ so that it cancels the rescaling coefficient.

Due to the rescaling, for the block of parameters $\{m_{hf}\}_{h\in[H], f\in[F]}$ we get $\gamma JVJ^T = O(1)$

Using Taylor expansion of $V$:
\begin{align}
    &V_{k_{hd}, k_{h'd'}} = \mathbb{E}[\partial_{k_{hd}}H(\mathbf{u})\partial_{k_{h'd'}}H(\mathbf{u})] = \mathbb{E}[\partial_{k_{hd}}H(D^{-\zeta}\tilde{m}, r, b, v)\partial_{k_{h'd'}}H(D^{-\zeta}\tilde{m}, r, b, v)] \\ &=\mathbb{E}\left[\left(\partial_{k_{hd}}H(0) + D^{-\zeta}\scprod{\tilde{m}}{\nabla_{\tilde{m}}\partial_{k_{hd}}H(m_1)}\right)\left(\partial_{k_{h'd'}}H(0) + D^{-\zeta}\scprod{\tilde{m}}{\nabla_{\tilde{m}}\partial_{k_{h'd'}}H(m_2)}\right)\right]\\
    &=\mathbb{E}[\partial_{k_{hd}}H(0)\partial_{k_{h'd'}}H(0)] \\
    &+ D^{-2\zeta}\mathbb{E}\partial_{k_{h'd'}}H(0)\scprod{\tilde{m}}{\nabla_{m}\partial_{k_{hd}}H(m_1)} + D^{-2\zeta}\mathbb{E}\partial_{k_{hd}}H(0)\scprod{\tilde{m}}{\nabla_{m}\partial_{k_{h'd'}}H(m_2)}\\
    &+ D^{-4\zeta}\mathbb{E}\scprod{\tilde{m}}{\nabla_{m}\partial_{k_{hd}}H(m_1)}\scprod{\tilde{m}}{\nabla_{m}\partial_{k_{h'd'}}H(m_2)}
\end{align}
It's easy to see that $\mathbb{E}\partial_{k_{hd}}H(0)\scprod{\tilde{m}}{\nabla_{m}\partial_{k_{h'd'}}H(m_2)}$ and $\mathbb{E}\scprod{\tilde{m}}{\nabla_{m}\partial_{k_{hd}}H(m_1)}\scprod{\tilde{m}}{\nabla_{m}\partial_{k_{h'd'}}H(m_2)}$ are of order $O(N_b^{-1})$ on the compact $E_K$. Indeed, $\|\tilde{m}\|$ is bounded by constant on $E_K$ and $\nabla_{m}\partial_{k_{hd}}H(m_1)$ is a random variable of the similar form as $\partial_{k_{hd}}H$, consisting of the sum of products of bounded functions of random variables and normal random variables $X_{ld}$.

We have $V_{k_{hd}, k_{h'd'}} = \mathbb{E}[\partial_{k_{hd}}H(0)\partial_{k_{h'd'}}H(0)] + O(D^{-2\zeta}N_b^{-1})$

Now, explicitly writing $\gamma JVJ^T$, we get:
\begin{align}
    \gamma(JVJ^T)_{m_{hf}, m_{h'f'}} &= \gamma D^{2\zeta}\sum_{k_{hd}, k_{h'd'}} V_{k_{hd}, k_{h'd'}} k^*_{fd}k^*_{f'd'} \\
    &= \gamma D^{2\zeta}\sum_{k_{hd}, k_{h'd'}} \mathbb{E}[\partial_{k_{hd}}H(0)\partial_{k_{h'd'}}H(0)] k^*_{fd}k^*_{f'd'} + \gamma O(D^{1}N_b^{-1}) \\
    &= \gamma D^{2\zeta}\sum_{k_{hd}, k_{h'd'}} \mathbb{E}[\partial_{k_{hd}}H(0)\partial_{k_{h'd'}}H(0)] k^*_{fd}k^*_{f'd'} + o(1).
\end{align}
Therefore, we get that $\mathbf{\Sigma}$ is constant matrix, independent of $\mathbf{u}$. 

Placing it all together, we have a system 
\begin{align}
    d\tilde{m}_{hf} &= -\partial_{m_{hf}}\Phi(\mathbf{u})dt + c_{\Sigma} dB_t \\
    \dot{r}_{hh'} &= - \partial_{r_{hh'}}\Phi(\mathbf{u})\\
    \dot{b}_{h} &= -\partial_{b_{h}}\Phi(\mathbf{u}) \\
    \dot{v} &= -\partial_v\Phi(\mathbf{u})
\end{align}
\end{proof}

\paragraph{Specialization phase Prop. \ref{res:2ndPhase}.}
To obtain the final sample complexity, we rely on Lemma \ref{res:hessI} on the Hessian of the loss at small $m$ and $r$ and unspecialized $b$. This lemma applies even after initialization, as we justify by the following assumptions.

We assume that $||\mathbb E\theta||_2$ is small enough so during the unspecialized phase the growth of $m$ remains bounded and we can consider the loss at $m\approx 0$. We moreover consider a small enough initialization $\eta$ so $r\approx 0$. By Lemma \ref{resApp:gradRI} on the gradient, $r=0$ is a fixed-point of the dynamics, and by Lemma \ref{res:hessI} on the Hessian it is a stable fixed-point for $m$ small enough. Consequently $r\approx 0$ holds even after initialization, until $m$ starts growing during the specialization phase. Numerically, we observe that $r=0$ is a stable fixed-point whose basin of attraction encompasses values of $m$ and $r$ of order one. Last, $b$ stays unspecialized until $m$ grows and specializes, because it corresponds to a stable or flat point by Lemma \ref{res:hessI}.

Lemma \ref{res:hessI} gives that there exists directions orthogonal to $\mathbb E\theta$ in the space of the features, that are descent directions and where the Hessian is not degenerated. More precisely, assuming the heads split evenly and that $m^\top\mathds{1}_H\approx0$, the descent directions are the eigenvectors of $\cov\theta$ projected in the orthogonal space, and the Hessian has strictly negative eigenvalues in all these directions.

The heads specialize because they tend to evolve orthogonally to $\mathds{1}_H$. In the space where $m^\top\mathds{1}_H\neq 0$ the Hessian has strictly larger eigenvalues and thus it requires more time to move in this direction.

\begin{lemma}[Sample complexity of weak recovery with SGD]
    Starting from initialization $m_{hf}=0$ when initial gradient $\nabla_m\Phi(m=0) = 0$, the sample complexity of reaching $m_{hf}=O(1)$ independent of $D$ with hight probability is $O(D^{1+2\zeta}\ln D)$ for arbitrary small $\zeta > 0$.
\end{lemma}
\begin{proof}
By lemma \ref{resApp:sgdLimit}, for $\tilde{m}_{hf}$ around zero, we can write 
$$d\tilde{m} \simeq -\nabla^2_{m}\Phi(0)\tilde{m}dt + c_{\Sigma}\mathbb{1}dB_t,$$
and projecting on the eigenvectors of $\nabla^2_{m}\Phi(0)$ with strictly negative eigenvalues (given by Lemma \ref{res:hessI}), we get for some $c_{v}>0$
$$d\scprod{v}{\tilde{m}} = c_{v}\scprod{v}{\tilde{m}}dt + c_{\Sigma, v}dBt.$$
This is a mean-repellent process, such that $\scprod{v}{\tilde{m}}(t) = c_{\Sigma, v}\int_{0}^te^{c_v(t-s)}dB_s$, i.e. $\mathbb{E}|\scprod{v}{\tilde{m}}(t)|=\Theta(e^{c_vt})$ and $\mathbb{E}\scprod{v}{\tilde{m}}(t)^2=\Theta(e^{2c_vt})$.

Now, by Paley–Zygmund inequality: $$\mathbb{P}(|\scprod{v}{\tilde{m}}(t)| > 0.5\mathbb{E}|\scprod{v}{\tilde{m}}(t)|)\geq 0.25\frac{\mathbb{E}[|\scprod{v}{\tilde{m}}(t)|^2]}{\mathbb{E}[|\scprod{v}{\tilde{m}}(t)|]^2} = \Theta(1).$$

Using weak convergence of sufficient statistics $m^{(D)}_{hf}(t)$ to the limit with continuous cumulative distribution function, we get uniform convergence of CDFs, and setting $t=\frac{2\zeta}{c_v}\ln D$ we get for $D$ large enough 
$$\mathbb{P}(|\scprod{v}{\tilde{m}_D}(t)| > 0.5\mathbb{E}|\scprod{v}{\tilde{m}}(t)|) =\mathbb{P}(D^{2\zeta}|\scprod{v}{m_D}(t)| > 0.5\Theta(D^{2\zeta}))  \geq \Theta(1).$$

Thus, we get that it takes $O(\ln D)$ to escape from the fixed point with constant probability. And the total time complexity is $O(D^{1+2\zeta}\ln D)$ for arbitrary small $\zeta > 0$. 
\end{proof}

\subsubsection{Derivatives of the loss}
\label{secApp:derivatives}
We derive the technical results about the gradient, the hessian and the 4th order derivative of the loss $\tilde{\mathcal{E}}$ in the space of the order parameters.
For the derivative of the activation function we use the notation
\begin{align}
  \partial_{h'\ell '}\sigma(\chi,b,v;h)_\ell = \frac{\partial}{\partial \chi_{h'\ell'}} \sigma(\chi,b,v;h)_\ell\ .
\end{align}

\begin{lemma}[Invariance of the unspecialized manifold by gradient descent.]
\label{resApp:invUnspMan}
Consider the unspecialized manifold $\mathcal{M}_\mathrm{u}$, where the heads are not specialized, defined by $m=\mathds{1}_H\tilde m^\top$ with $\tilde m\in\mathbb R^F$, $r=\tilde r_1I_H+\tilde r_2\mathds{1}_H\mathds{1}_H^\top$ with $\tilde r_1>0,\tilde r_1+H\tilde r_2>0$ and $b=\tilde b\mathds{1}_H$ with $\tilde b\in\mathbb R$. $\mathcal{M}_\mathrm{u}$ is invariant by the gradient descent eq.~\eqref{eq:dynEff}.
\end{lemma}
\begin{proof}
The loss is invariant by permutation of the heads and therefore on $\mathcal{M}_\mathrm{u}$ $\nabla_{m_h}\tilde{\mathcal{E}}_\sigma$, $\partial_{r_{hh}}\tilde{\mathcal{E}}_\sigma$, $\partial_{r_{h\neq h'}}\tilde{\mathcal{E}}_\sigma$ and $\partial_{b_h}\tilde{\mathcal{E}}_\sigma$ do not depend on $h$.
\end{proof}

\begin{lemma}[Gradient of the loss at initialization and in the unspecialized phase]
\label{resApp:gradI}
Let $\mathbb E\theta=(\mathbb E_{P_\theta}\,\theta_f)_{f\in[F]}\in\mathbb R^F$ be the mean of the signal. Take $m$, $r$ and $b$ on the unspecialized manifold $\mathcal{M}_\mathrm{u}$. Take all the heads aligned with $\mathbb E\theta$, i.e. take $x\in\mathbb R$ and $m_h=x\mathbb E\theta$ for all $h$. There is $c^{(1)}(x,r,b,v)\in\mathbb R$ such that for all $h$
\begin{align}
\nabla_{m_h}\tilde{\mathcal{E}}_\sigma(m,r,b,v) &= -c^{(1)}(x,r,b,v)\mathbb E\theta\ .
\end{align}
Moreover at initialization $c^{(1)}(0,r,0,1)>0$.
\end{lemma}
\begin{proof}
We compute the gradient of the loss in the space of the order parameters. We remind that the reparameterized loss is
\begin{align}
\tilde{\mathcal{E}}_\sigma(m,r,b,v) &= \mathbb E_{\epsilon,\theta,\chi^*,\xi}\left[\sum_\ell^L\left(\delta_{\ell,\epsilon}-\frac{1}{H}\sum_{h'}^H\sigma(\chi,b,v;h')_\ell\right)^2\right] \\
\chi_{h\ell} &= \sum_f^Fm_{hf}\chi_{f\ell}^*+\sum_{h'}^Hr_{hh'}\xi_{h'}\ ,\quad h\in[H], \ell\in[L]\ .
\end{align}
with $\epsilon\sim\unif(\{1,\ldots,L\})$, $\theta\sim P_\theta$ and conditionally on $\epsilon$ and $\theta$, $\chi_{:,\ell}^*\sim\mathcal N(\delta_{\ell,\epsilon}\theta,I_F)$ and $\xi_{:,\ell}\sim\mathcal N(0,I_H)$ for $\ell\in[L]$. The gradient is
\begin{align}
\nabla_{m_h}\tilde{\mathcal{E}}_\sigma(m,r,b,v) &= 2\mathbb E_{\epsilon,\theta,\chi^*,\xi}\sum_\ell^L\left(\frac{1}{H}\sum_{h'}^H\sigma(\chi,b,v;h')_\ell-\delta_{\ell,\epsilon}\right) \sum_{h',\ell'}^{H,L}\partial_{h\ell '}\sigma(\chi,b,v;h')_\ell \chi_{:, \ell'}^*.
\end{align}

We show the gradient is collinear to $\mathbb E\theta$ if all heads align with $\mathbb{E}\theta$, i.e. $m_h = x\mathbb{E} \theta$ for all $h\in [H]$. Let $w\in\mathbb R^F$ be orthogonal to $\mathbb E\theta$.
\begin{align}
w^\top\nabla_{m_h}\tilde{\mathcal{E}}_\sigma(m,r,b,v) &= 2\mathbb E_{\epsilon,\theta,\chi,\xi}\sum_\ell^L\left(\frac{1}{H}\sum_{h'}^H\sigma(\chi,b,v;h')_\ell-\delta_{\ell,\epsilon}\right) \sum_{h',\ell'}^{H,L}\partial_{h\ell '}\sigma(\chi,b,v;h')_\ell w^\top\chi_{:,\ell'}^*
\end{align}
Then by orthogonality $w^\top\chi_{:\ell'}^*$ and $\chi_{h\ell}=x(\mathbb E\theta)^\top\chi_{:\ell}^*+\ldots$ are independent Gaussian random variables for all $\ell, \ell'$. Thus, we can factorize the expectation. Since $\mathbb E\,w^\top\chi_{:\ell'}^*=w^\top\mathbb E\theta=0$, we have $w^\top\nabla_{m_h}\tilde{\mathcal{E}}_\sigma(m,r,b,v)=0$. Moreover we consider the unspecialized manifold, and so there is a same $c^{(1)}(x,r,b,v)\in\mathbb R$ for all $h$ such that $\nabla_{m_h}\tilde{\mathcal{E}}_\sigma(m,r,b,v) = -c^{(1)}(x,r,b,v)\mathbb E\theta$.

At initialization $b=0$, $v=1$ and $m=0$, and we have $\chi_{h'}$ independent of $\{\chi^*_h\}_{h\in[H]}$ for all $h'$; thus
\begin{align}
\nabla_{m_h}\tilde{\mathcal{E}}_\sigma(0,r,0,1) &= 2\mathbb E_{\epsilon,\xi}\sum_\ell^L\left(\frac{1}{H}\sum_{h'}^H\sigma(\chi,0,1;h')_\ell-\delta_{\ell,\epsilon}\right) \sum_{h',\ell'}^{H,L}\partial_{h\ell '}\sigma(\chi,0,1;h')_\ell \mathbb E_{\theta,\chi}\chi_{:, \ell'}^* \\
&= 2\mathbb E_{\epsilon,\xi}\sum_\ell^L\left(\frac{1}{H}\sum_{h'}^H\sigma(\chi,0,1;h')_\ell-\delta_{\ell,\epsilon}\right) \sum_{h'}^{H}\partial_{h\epsilon}\sigma(\chi,0,1;h')_\ell \mathbb E\theta \\
&= 2\mathbb E_{\epsilon,\xi}\left(\frac{1}{H}\sum_{h'}^H\sigma(\chi,0,1;h')_\epsilon-1\right) \sum_{h'}^{H}\partial_{h\epsilon}\sigma(\chi,0,1;h')_\epsilon \mathbb E\theta \\
&\qquad {}+ 2\mathbb E_{\epsilon,\xi}\sum_{\ell\neq\epsilon}^L\frac{1}{H}\sum_{h'}^H\sigma(\chi,0,1;h')_\ell \sum_{h'}^{H}\partial_{h\epsilon}\sigma(\chi,0,1;h')_\ell\mathbb E\theta \nonumber
\end{align}
The two pre-factors in front of $\mathbb E\theta$ are negative because
\begin{align}
\frac{1}{H}\sum_{h'}^H\sigma(\chi,0,1;h')_\epsilon-1 &< 0 & \sum_{h'}^{H}\partial_{h\epsilon}\sigma(\chi,0,1;h')_\epsilon &> 0  \\
\frac{1}{H}\sum_{h'}^H\sigma(\chi,0,1;h')_\ell &> 0 & \sum_{h'}^{H}\partial_{h\epsilon}\sigma(\chi,0,1;h')_\ell &<0
\end{align}
for all $\epsilon\neq\ell$, $h$, $\chi$ and for the three different activation functions $\sigma$. Consequently $c^{(1)}(0,r,0,1)>0$.
\end{proof}

\begin{lemma}[Gradient of the loss with respect to $r$ at small $r$]
\label{resApp:gradRI}
Consider $r=0$, then for all $m$
\begin{align}
 \nabla_r\tilde{\mathcal{E}}_\sigma(m,0,b,v) = 0 .
\end{align}
\end{lemma}
\begin{proof}
The loss is a symmetric function of $r$: it is invariant by the change of variables $(r,\xi)\mapsto(-r,-\xi)$. Therefore its gradient is null at $r=0$.
\end{proof}

\begin{lemma}[Gradient of the loss with respect to $b$ and $v$ at small $m$ and $r$ for the softmax-1]
\label{resApp:gradBVI}
Consider $m=0$, $r=0$ and $\sigma$ to be the softmax-1. Take unspecialized heads i.e. $b=\tilde b\mathds{1}_H$ for $\tilde b\in\mathbb R$. Then
\begin{align}
\nabla_{\tilde b}\tilde{\mathcal{E}}_\sigma(0,0,b,v) &= -2\left(\frac{Lv}{L+e^{\tilde b}}-1\right) \frac{v}{H}\frac{e^{\tilde b}}{(L+e^{\tilde b})^2} \\
\nabla_{v}\tilde{\mathcal{E}}_\sigma(0,0,b,v) &= 2\left(\frac{Lv}{L+e^{\tilde b}}-1\right)\frac{1}{L+e^{\tilde b}}
\end{align}
The fixed-points of this system satisfy $Lv=L+e^{\tilde b}$ and they are attractive. Last at initialization $b=0$ and $v=1$ and
\begin{align}
\nabla_{\tilde b}\tilde{\mathcal{E}}_\sigma(0,0,0,1) &>0 \\
\nabla_{v}\tilde{\mathcal{E}}_\sigma(0,0,0,1) &<0
\end{align}
\end{lemma}
\begin{proof}
The proof is a straightforward computation of the derivatives.
\end{proof}

\begin{lemma}[Hessian of the loss at small $m$ and $r$]
\label{resApp:hessI}
Consider $m\in\mathbb R^{H\times F}$ in the space orthogonal to $\mathbb E\theta$ i.e. $m_h^\top\mathbb E\theta=0$ for all $h$. Take $r\in\mathcal S_+^H$. Assume that $b$ is not specialized, i.e. $b_h=\tilde b$ for all $h$. For the softmax-1 assume that $b$ and $v$ reached the fixed-point described by Lemma \ref{resApp:gradBVI} i.e. $Lv=L+e^{\tilde b}$. The loss around $m\approx 0$, $r\approx 0$, $b$ and $v$ can be expanded as
\begin{align}
& \tilde{\mathcal{E}}_\sigma(m,r,b+\bar b,v) = \tilde{\mathcal{E}}_\sigma(0,0,b,v) + c_1^{(2)} \sum_{h,h'}^H(r^2)_{hh'} + c_5^{(2)}(\bar b^\top\mathds{1}_H)^2 \\
&\qquad {}+ \left(\mathds{1}_H\mathds{1}_H^\top\otimes(c_1^{(2)}I_F+(c_2^{(2)}+c_4^{(2)})\cov\theta)-(c_3^{(2)}+c_4^{(2)})I_H\otimes\cov\theta\right)\cdot(m,m) \nonumber \\
&\qquad {}+ \mathcal O((||r||_F^2+||b||_2^2+||m||_F^2)^2) \nonumber
\end{align}
with $\otimes$ the tensorial product between the spaces $\mathbb R^H$ and $\mathbb R^F$, and
$c_1^{(2)}, c^{(2)}_2, c^{(2)}_3\in\mathbb R$ strictly positive and $c^{(2)}_4\in\mathbb R, c^{(2)}_5\in\mathbb R$ positive for all $L\geq 3$.
\end{lemma}

\begin{proof}
$\nabla^2_{r,m}\tilde{\mathcal{E}}_\sigma(0,0,b,v)=0$ because the gradient w.r.t. $r$ brings a $\xi$ while the gradient w.r.t. $m$ brings a $\chi^*$. $\xi$ is centered and independent of $\chi^*$ thus the expectation is null. The same reasoning holds for $\nabla^2_{r,b}\tilde{\mathcal{E}}_\sigma(0,0,b,v)=0$. We also have $\nabla^2_{m,b}\tilde{\mathcal{E}}_\sigma(0,0,b,v) \cdot (m,b)=0$ because we consider $m\mathbb E\theta=0$.

For the Hessian w.r.t. $r$ we expand the loss around $\chi=r\xi\approx 0$ :
{\small
\begin{align}
& \tilde{\mathcal{E}}_\sigma(0,r,b,v) = \mathbb E_{\epsilon,\xi}\left[\sum_\ell^L\left(\delta_{\ell,\epsilon}-\frac{1}{H}\sum_{h}^H\sigma(r\xi,b,v;h)_\ell\right)^2\right] \\
&\quad = \mathbb E_{\epsilon,\xi}\left[\sum_\ell^L\left(\delta_{\ell,\epsilon}-\frac{1}{H}\sum_{h}^H\left(\sigma(0,b,v;h)_\ell+\sum_{h',\ell'}\partial_{h'\ell'}\sigma(0,b,v;h)_\ell r_{h'}^\top\xi_{:\ell'} \right.\right.\right.\\
&\qquad \left.\left.\left.{}+\frac{1}{2}\sum_{h',\ell',h'',\ell''}\partial^2_{h'\ell',h''\ell''}\sigma(0,b,v;h)_\ell r_{h'}^\top\xi_{:\ell'}r_{h''}^\top\xi_{:\ell''}\right)\right)^2\right] \nonumber \\
&\quad = \mathbb E_{\epsilon,\xi}\left[\sum_\ell^L\left(\delta_{\ell,\epsilon}-H^{-1}\sum_h\sigma(0,b,v;h)_\ell\right)^2 - \sum_\ell^L\frac{1}{H}\left(\delta_{\ell,\epsilon}-H^{-1}\sum_h\sigma(0,b,v;h)_\ell\right) \right.\\
&\quad \times \mkern-6mu\sum_{h,h',\ell',h'',\ell''}\partial^2_{h'\ell',h''\ell''}\sigma(0,b,v;h)_\ell r_{h'}^\top\xi_{:\ell'}r_{h''}^\top\xi_{:\ell''} + \left.\sum_\ell^L\frac{1}{H^2}\left(\sum_{h,h',\ell'}\partial_{h'\ell'}\sigma(0,b,v;h)_\ell r_{h'}^\top\xi_{:\ell'}\right)^2\right] \nonumber \\
&\quad = \tilde{\mathcal{E}}_\sigma(0,0,b,v) + \sum_\ell^L\frac{1}{H}\left(H^{-1}\sum_h\sigma(0,b,v;h)_\ell-\delta_{\ell,1}\right)\sum_{h,h',h'',\ell'}\partial^2_{h'\ell',h''\ell'}\sigma(0,b,v;h)_\ell r_{h'}^\top r_{h'} \nonumber\\
&\qquad {}+ \sum_\ell^L\frac{1}{H^2}\sum_{h_1,h_1',h_2,h_2',\ell'}\partial_{h_1'\ell'}\sigma(0,b,v;h_1)_\ell\partial_{h_2'\ell'}\sigma(0,b,v;h_2)_\ell r_{h_1'}^\top r_{h_2'} \\
&\quad = \tilde{\mathcal{E}}_\sigma(0,0,b,v) + \frac{1}{H^2}\sum_{h',h''}\sum_{\ell,\ell'}\left(\sum_{h}\partial_{h'\ell'}\sigma(0,b,v;h)_\ell\right)\left(\sum_{h}\partial_{h''\ell'}\sigma(0,b,v;h)_\ell\right) r_{h'}^\top r_{h''}
\end{align}
}
where we took the expectation over $\xi_\ell\sim\mathcal N(0,I_H)$ and discarded the term of order one in $r$ with null expectation. We simplified $\sum_\ell^L\left(H^{-1}\sum_h\sigma(0,b,v;h)_\ell-\delta_{\ell,1}\right)\times(\mathrm{function\ independent\ of\ }\ell)=0$, using that $b$ is not specialized and using the fixed-point condition for $b$ and $v$ for the softmax-1. Consequently,
\begin{align}
\tilde{\mathcal{E}}_\sigma(0,r,b,v) &= \tilde{\mathcal{E}}_\sigma(0,0,b,v) + c_1^{(2)}\sum_{h',h''}(r^2)_{h'h''}
\end{align}
with the constant
\begin{align}
c_1^{(2)} &= \frac{1}{H^2}\sum_{\ell,\ell'}\left(\sum_{h}\partial_{1\ell'}\sigma(0,b,v;h)_\ell\right)^2 >0\ .
\end{align}
We chose the particular index 1 for the derivative because of the permutation invariance with respect to the heads.

For the Hessian w.r.t. $m$ we perform a similar derivation, expanding the loss around $\chi=m\chi^*\approx 0$. We take in account the fact that $\chi^*_{f\ell}\sim\mathcal N(\delta_{\ell,\epsilon}\theta_f,1)$.
{\small
\begin{align}
& \tilde{\mathcal{E}}_\sigma(m,0,b,v) = \mathbb E_{\epsilon,\theta,\chi^*}\left[\sum_\ell^L\left(\delta_{\ell,\epsilon}-\frac{1}{H}\sum_{h}^H\sigma(m\chi^*,b,v;h)_\ell\right)^2\right] \\
& = \mathbb E_{\epsilon,\theta,\chi^*}\left[\sum_\ell^L\left(\delta_{\ell,\epsilon}-\frac{1}{H}\sum_{h}^H\left(\sigma(0,b,v;h)_\ell+\sum_{h',\ell'}\partial_{h'\ell'}\sigma(0,b,v;h)_\ell m_{h'}^\top\chi^*_{:\ell'}\right.\right.\right. \\
&\quad\left.\left.\left.{}+\frac{1}{2}\sum_{h',\ell',h'',\ell''}\partial^2_{h'\ell',h''\ell''} \sigma(0,b,v;h)_\ell m_{h'}^\top\chi^*_{:\ell'}m_{h''}^\top\chi^*_{:\ell''}\right)\right)^2\right] \nonumber \\
& = \mathbb E_{\epsilon,\theta,\chi^*}\left[\sum_\ell^L\left(\delta_{\ell,\epsilon}-H^{-1}\sum_h\sigma(0,b,v;h)_\ell\right)^2 + \sum_\ell^L\frac{1}{H^2}\left(\sum_{h,h',\ell'}\partial_{h'\ell'}\sigma(0,b,v;h)_\ell m_{h'}^\top\chi^*_{:\ell'}\right)^2\right. \nonumber \\
&\quad\left.{}- 2\sum_\ell^L\frac{1}{H}\left(\delta_{\ell,\epsilon}-H^{-1}\sum_h\sigma(0,b,v;h)_\ell\right)\sum_{h,h',\ell'}\partial_{h'\ell'}\sigma(0,b,v;h)_\ell m_{h'}^\top\chi^*_{:\ell'}  \right. \\
&\quad\left. {}- \sum_\ell^L\frac{1}{H}\left(\delta_{\ell,\epsilon}-H^{-1}\sum_h\sigma(0,b,v;h)_\ell\right)\sum_{h,h',\ell',h'',\ell''}\partial^2_{h'\ell',h''\ell''}\sigma(0,b,v;h)_\ell m_{h'}^\top\chi^*_{:\ell'}m_{h''}^\top\chi^*_{:\ell''} \right] \nonumber \\
& = \tilde{\mathcal{E}}_\sigma(0,0,b,v) + 2\sum_\ell^L\frac{1}{H}\left(H^{-1}\sum_h\sigma(0,b,v;h)_\ell-\delta_{\ell,1}\right)\sum_{h,h'}\partial_{h'1}\sigma(0,b,v;h)_\ell m_{h':}^\top\mathbb E\theta \\
&{}+ \sum_\ell^L\frac{1}{H}\mkern-5mu \left(\mkern-5mu H^{-1}\sum_h\sigma(0,b,v;h)_\ell-\delta_{\ell,1}\mkern-5mu \right)\mkern-8mu \sum_{h,h',\ell',h''}\mkern-15mu \partial^2_{h'\ell',h''\ell'}\sigma(0,b,v;h)_\ell \sum_{f,f'}^Fm_{h'f}m_{h''f'}(\delta_{f,f'}+\delta_{\ell',1}\mathbb E\theta_f\theta_{f'}) \nonumber \\
&{}+ \sum_\ell^L\frac{1}{H^2}\sum_{h_1,h_1',h_2,h_2',\ell'}\partial_{h_1'\ell'}\sigma(0,b,v;h_1)_\ell\partial_{h_2'\ell'}\sigma(0,b,v;h_2)_\ell \sum_{f,f'}^Fm_{h_1'f} m_{h_2'f'}(\delta_{f,f'}+\delta_{\ell',1}\mathbb E\theta_f\theta_{f'}) \nonumber \\
& = \tilde{\mathcal{E}}_\sigma(0,0,b,v) + \sum_\ell^L\frac{1}{H}\mkern-5mu \left(\mkern-5mu H^{-1}\sum_h\sigma(0,b,v;h)_\ell-\delta_{\ell,1}\mkern-5mu \right)\mkern-8mu \sum_{h,h',h''}\mkern-8mu \partial^2_{h'1,h''1}\sigma(0,b,v;h)_\ell \sum_{f,f'}^Fm_{h'f}m_{h''f'}\mathbb E\theta_f\theta_{f'} \\
&{}+ \frac{1}{H^2}\sum_{h,h'}\sum_{\ell,\ell'}\left(\sum_h\partial_{h'\ell'}\sigma(0,b,v;h)_\ell\right)\left(\sum_h\partial_{h''\ell'}\sigma(0,b,v;h)_\ell\right) \sum_{f,f'}^Fm_{h'f} m_{h''f'}(\delta_{f,f'}+\delta_{\ell',1}\mathbb E\theta_f\theta_{f'}) \nonumber
\end{align}
}
We consider the space orthogonal to $\mathbb E\theta$, i.e. $m_{h:}^\top\mathbb E\theta=0$ for all $h$. We simplified $\sum_\ell^L\left(H^{-1}\sum_h\sigma(0,b,v;h)_\ell-\delta_{\ell,1}\right)\times(\mathrm{function\ independent\ of\ }\ell)=0$. We introduce the covariance of $\theta$. Consequently,
\begin{align}
\tilde{\mathcal{E}}_\sigma(m,0,b,v) &= \tilde{\mathcal{E}}_\sigma(0,0,b,v) + c^{(2)}_1\sum_{h',h''}m_{h'}^\top m_{h''} + (c^{(2)}_2+c^{(2)}_4)\sum_{h',h''}m_{h'}^\top\cov(\theta) m_{h''} \nonumber\\
&\qquad {}- (c^{(2)}_3+c^{(2)}_4)\sum_{h'}m_{h'}^\top\cov(\theta) m_{h'}\\
c_1^{(2)} &= \frac{1}{H^2}\sum_{\ell,\ell'}\left(\sum_{h}\partial_{1\ell'}\sigma(0,b,v;h)_\ell\right)^2 >0\\
c^{(2)}_2 &= \frac{1}{H^2}\sum_{\ell,\ell'}\left(\sum_{h}\partial_{11}\sigma(0,b,v;h)_\ell\right)^2 >0 \\
c^{(2)}_3 &= -\frac{1}{H}\sum_\ell^L\left(H^{-1}\sum_h\sigma(0,b,v;h)_\ell-\delta_{\ell,1}\right)\sum_{h}\partial^2_{11,11}\sigma(0,b,v;h)_\ell \\
c^{(2)}_4 &= \frac{1}{H}\sum_\ell^L\left(H^{-1}\sum_h\sigma(0,b,v;h)_\ell-\delta_{\ell,1}\right)\sum_{h}\partial^2_{11,21}\sigma(0,b,v;h)_\ell
\end{align}
so the loss is
\begin{align}
& \tilde{\mathcal{E}}_\sigma(m,0,b,v) = \tilde{\mathcal{E}}_\sigma(0,0,b,v) \\
& \qquad\qquad{}+\left(\mathds{1}_H\mathds{1}_H^\top\otimes(c_1^{(2)}I_F+(c_2^{(2)}+c_4^{(2)})\cov\theta)-(c_3^{(2)}+c_4^{(2)})I_H\otimes\cov\theta\right)\cdot(m,m)\ .\nonumber
\end{align}
Between $c_3^{(2)}$ and $c_4^{(2)}$ we distinguished the cases $h'=h''$ and $h'\neq h''$. We compute these two constants for the different activation functions. We have\\
\begin{center}
\begin{tabular}{ccc}
\hline\hline
$\sigma$ & $c_3^{(2)}$ & $c_4^{(2)}$ \\
\hline
softmax & $H^{-1}(L-1)L^{-2}(1-2L^{-1})$ & 0 \\
softmax-1 & $vH^{-1}(L-1)L^{-1}(L+e^{\tilde b})^{-1}(1-2(L+e^{\tilde b})^{-1})$ & 0 \\
B-softmax & $H^{-1}(L-1)L^{-2}(1-2(HL)^{-1})$ & $(L-1)L^{-1}2(HL)^{-2}$ \\
\hline\hline
\end{tabular}
\end{center}
Consequently for all $L\geq 3$ one has $c_3^{(2)}>0$ and $c_4^{(2)}\geq 0$.

For the Hessian w.r.t. $b$ we compute that for the softmax-1
\begin{align}
\partial^2_{b_h,b_h'}\tilde{\mathcal{E}}_\sigma(0,0,b,v) &= 2\frac{v^2L}{H^2}\frac{e^{2\tilde b}}{(L+e^{\tilde b})^4}
\end{align}
and for the B-softmax
\begin{align}
\partial^2_{b_h,b_h'}\tilde{\mathcal{E}}_\sigma(0,0,b,v) &= 0\ .
\end{align}
\end{proof}

\begin{lemma}[Hessian of the loss at small but finite $m$]
\label{resApp:diff4}
We take $\sigma$ softmax. We assume that $\mathbb E\theta=0$ and that the $\theta_f$ are independent. Let $n$ be an integer, $\bar f_1,\ldots,\bar f_n\in[F]^{n}$ all different, $\bar m\in\mathbb R^{H\times F}$ and assume that $\bar m_{:f}=0$ for all $f\notin\{\bar f_1,\ldots,\bar f_n\}$. Pick $f\notin\{\bar f_1,\ldots,\bar f_n\}$; then the Hessian of the loss is
\begin{align}
\nabla_{m_{:f},m_{:f}}^2\tilde{\mathcal{E}}_\sigma(\bar m,0,0,0) &= \nabla_{m_{:f},m_{:f}}^2\tilde{\mathcal{E}}_\sigma(0,0,0,0) +
\sum_i^nc_{1,i}^{(4)}||\bar m_{:\bar f_i}||_2^2 I_H \\
& {}+\sum_i^nc^{(4)}_{2,i}\mathrm{Diag}(\bar m_{:\bar f_i}^{\odot 2}) + \sum_i^nc^{(4)}_{3,i}\bar m_{:\bar f_i}\bar m_{:\bar f_i}^\top + M+\mathcal O(||\bar m||_F^4) \nonumber
\end{align}
where $M$ is a quadratic form that cancels when $m_{:f}^\top\mathds{1}_H=0$ and $\bar m^\top\mathds{1}_H=0$, and where $c^{(4)}_{1,i}>0, c^{(4)}_{2,i}\in\mathbb R, c^{(4)}_{3,i}>0$ for all $L\geq 3$ does not depend on $\bar m$.
\end{lemma}

\begin{proof}
The proof is a Taylor expansion of the loss to the 2nd order in $m_f$ around $\bar m$ and to the 2nd order in $\bar m$ around 0. We take $m\in\mathbb R^{H\times F}$ with $m_{\bar f_i}=0$ for all $\bar f_i$; it encompasses the case of a matrix where only the $f$-th column is not null equal to $m_f$. Since we consider $\sigma$ softmax, to lighten the notation we write $\sigma(\chi_h)$ for $\sigma(\chi,b,v;h)$.
{\small
\begin{align}
& \tilde{\mathcal{E}}_\sigma(m+\bar m,0,0,0) = \mathbb E_{\epsilon,\theta,\chi^*}\left[\sum_\ell^L\left(\delta_{\ell,\epsilon}-\frac{1}{H}\sum_{h}^H\sigma(m_{h}^\top\chi^*)_\ell\right)^2\right] \\
& = \mathbb E_{\epsilon,\theta,\chi^*}\left[\sum_\ell^L\left(\delta_{\ell,\epsilon}-\frac{1}{H}\sum_{h}^H\left(\sigma(\bar m_{h}^\top\chi^*)_\ell+\sum_{\ell'}\partial_{\ell'}\sigma(\bar m_{h}^\top\chi^*)_\ell m_{h}^\top\chi^*_{:\ell'} \right.\right.\right.\\
&\qquad\left.\left.\left.{}+\frac{1}{2}\sum_{\ell',\ell''}\partial^2_{\ell'\ell''}\sigma(\bar m_{h}^\top\chi^*)_\ell m_{h}^\top\chi^*_{:\ell'}m_{h}^\top\chi^*_{:\ell''}\right)\right)^2\right] \nonumber \\
& = \mathbb E_{\epsilon,\theta,\chi^*}\left[\sum_\ell^L\left(\delta_{\ell,\epsilon}-\frac{1}{H}\sum_{h}^H\sigma(\bar m_{h}^\top\chi^*)_\ell\right)^2 + \sum_\ell^L\frac{1}{H^2}\left(\sum_{h,\ell'}\partial_{\ell'}\sigma(\bar m_{h}^\top\chi^*)_\ell m_{h}^\top\chi^*_{:\ell'}\right)^2 \right. \\
&\quad {}- 2\sum_\ell^L\left(\delta_{\ell,\epsilon}-\frac{1}{H}\sum_{h}^H\sigma(\bar m_{h}^\top\chi^*)_\ell\right)\frac{1}{H}\sum_{h,\ell'}\partial_{\ell'}\sigma(\bar m_{h}^\top\chi^*)_\ell m_{h}^\top\chi^*_{:\ell'} \nonumber \\
&\quad \left.{}- \sum_\ell^L\left(\delta_{\ell,\epsilon}-\frac{1}{H}\sum_{h}^H\sigma(\bar m_{h}^\top\chi^*)_\ell\right)\frac{1}{H}\sum_{h,\ell',\ell''}\partial^2_{\ell'\ell''}\sigma(\bar m_{h}^\top\chi^*)_\ell m_{h}^\top\chi^*_{:\ell'}m_{h}^\top\chi^*_{:\ell''} \right] \nonumber \\
& = \tilde{\mathcal{E}}_\sigma(\bar m,0,0,0) + -2\mathbb E_{\epsilon,\theta,\chi^*}\sum_\ell^L\left(\delta_{\ell,\epsilon}-\frac{1}{H}\sum_{h}^H\sigma(\bar m_{h}^\top\chi^*)_\ell\right)\frac{1}{H}\sum_{h}\partial_{1}\sigma(\bar m_{h}^\top\chi^*)_\ell \underbrace{m_{h}^\top\mathbb E\theta}_{=0} \\
&\quad \underbrace{{}-\mathbb E_{\epsilon,\theta,\chi^*}\sum_\ell^L\left(\delta_{\ell,\epsilon}-\frac{1}{H}\sum_{h}^H\sigma(\bar m_{h}^\top\chi^*)_\ell\right)\frac{1}{H}\sum_{h,\ell'}\partial^2_{\ell'\ell'}\sigma(\bar m_{h}^\top\chi^*)_\ell\sum_{f,f'}^Fm_{hf}m_{hf'}(\delta_{f,f'}+\delta_{\ell',1}\mathbb E\theta_f\theta_{f'})}_{(1)} \nonumber \\
&\quad {}+ \underbrace{\mathbb E_{\epsilon,\theta,\chi^*}\sum_\ell^L\frac{1}{H^2}\sum_{h,h',\ell'}\partial_{\ell'}\sigma(\bar m_{h}^\top\chi^*)_\ell\partial_{\ell'}\sigma(\bar m_{h'}^\top\chi^*)_\ell\sum_{f,f'}^Fm_{hf} m_{h'f'}(\delta_{f,f'}+\delta_{\ell',1}\mathbb E\theta_f\theta_{f'})}_{(2)} \nonumber
\end{align}
}
where we used the independence of the $\theta_f$ to factorize the expectation. We expand with respect to $\bar m$, discarding the 1st order terms because $\mathbb E\theta=0$.
{\small
\begin{align}
(1) &= c_2^{(2)}\sum_hm_{h}^\top\cov(\theta)m_{h} \\
&\quad {}+ \mathbb E_{\epsilon,\theta,\chi^*}\sum_\ell^L\frac{1}{2H}\mkern-3mu\sum_{h',\ell'',\ell'''}\mkern-6mu\partial^2_{\ell''\ell'''}\sigma(0)_\ell\bar m_{h'}^\top\chi^*_{:\ell''}\bar m_{h'}^\top\chi^*_{:\ell'''}\frac{1}{H}\sum_{h,\ell'}\partial^2_{\ell'\ell'}\sigma(0)_\ell m_{h}^\top(I_F+\delta_{\ell',1}\cov(\theta))m_{h} \nonumber \\
&\quad {}+\mathbb E_{\epsilon,\theta,\chi^*}\sum_\ell^L\frac{1}{H}\underbrace{\sum_{h',\ell'''}\partial_{\ell'''}\sigma(0)_\ell\bar m_{h'}^\top}_{=0}\chi^*_{:\ell'''}\frac{1}{H}\sum_{h,\ell',\ell''}\partial^3_{\ell'\ell'\ell''}\sigma(0)_\ell\bar m_{h}^\top\chi^*_{:\ell''} m_{h}^\top(I_F+\delta_{\ell',1}\cov(\theta))m_{h} \nonumber \\
&\quad {}+\mathbb E_{\epsilon,\theta,\chi^*}\sum_\ell^L(\sigma(0)_\ell-\delta_{\ell,1})\frac{1}{2H}\sum_{h,\ell',\ell'',\ell'''}\mkern-10mu \partial^4_{\ell'\ell'\ell''\ell'''}\sigma(0)_\ell\bar m_{h}^\top\chi^*_{:\ell''}\bar m_{h}^\top\chi^*_{:\ell'''} m_{h}^\top(I_F+\delta_{\ell',1}\cov(\theta))m_{h} \nonumber \\
&= c_2^{(2)}\sum_hm_{h}^\top\cov(\theta)m_{h} \\
&\quad{}+ \sum_\ell^L\frac{1}{2H^2}\sum_{h',\ell''}\partial^2_{\ell''\ell''}\sigma(0)_\ell\bar m_{h'}^\top(I_F+\delta_{\ell'',1}\cov(\theta))\bar m_{h'}\sum_{h,\ell'}\partial^2_{\ell'\ell'}\sigma(0)_\ell m_{h}^\top(I_F+\delta_{\ell',1}\cov(\theta))m_{h} \nonumber \\
&\quad {}+\mkern-5mu\sum_\ell^L(\sigma(0)_\ell-\delta_{\ell,1})\mkern-5mu\sum_{\ell',\ell''}\frac{1}{2H}\mkern-3mu\sum_{h}\partial^4_{\ell'\ell'\ell''\ell''}\sigma(0)_\ell \bar m_{h}^\top(I_F+\delta_{\ell'',1}\cov(\theta))\bar m_{h} m_{h}^\top(I_F+\delta_{\ell',1}\cov(\theta))m_{h} \nonumber \\
&= c_2^{(2)}\sum_hm_{h}^\top\cov(\theta)m_{h} + \sum_\ell^L\frac{1}{2H^2}\sum_{h'}\partial^2_{1,1}\sigma(0)_\ell\bar m_{h'}^\top\cov(\theta)\bar m_{h'}\sum_{h}\partial^2_{1,1}\sigma(0)_\ell m_{h}^\top\cov(\theta)m_{h} \nonumber \\
&\quad {}+\sum_\ell^L(\sigma(0)_\ell-\delta_{\ell,1})\sum_{\ell',\ell''}\frac{1}{2H}\sum_{h}\partial^4_{\ell'\ell'\ell''\ell''}\sigma(0)_\ell \bar m_{h}^\top(I_F+\delta_{\ell'',1}\cov\theta)\bar m_{h} m_{h}^\top(I_F+\delta_{\ell',1}\cov\theta)m_{h}
\end{align}
}
where we discarded a term because of $\bar m_{:f}^\top\mathds{1}_H=0$ and used that $\sum_{\ell'}\partial^2_{\ell'\ell'}\sigma(0)_\ell=0$ for all $\ell$. Consequently there are constants $\tilde c_{1,i}^{(4)}>0, \tilde c^{(4)}_{2,i}, \tilde c^{(4)}_{3,i}\in\mathbb R$ for all $L\geq 3$ independent of $\bar m$ such that
\begin{align}
(1) &= c_2^{(2)} (I_H\otimes\cov\theta) \cdot(m, m) + \sum_i^n\tilde c_{1,i}^{(4)}|\bar m_{:\bar f_i}||_2^2 \left(I_H\otimes\cov\theta\right) \cdot(m, m) \nonumber \\
& \quad {}+\sum_i^n\mathrm{Diag}(\bar m_{:\bar f_i}^{\odot 2})\otimes\left(\tilde c^{(4)}_{2,i}I_F+\tilde c^{(4)}_{3,i}\cov\theta\right)\cdot(m, m)\ .
\end{align}
Turning to the second term,
\begin{align}
(2) &= c_1^{(2)}\sum_{h,h'}m_{h}^\top m_{h'}+c_3^{(2)}\sum_{h,h'}m_{h}^\top\cov(\theta)m_{h'} \\
& {}+\sum_\ell^L\frac{1}{H^2}\mkern-6mu\sum_{h,h',\ell',\ell''}\mkern-6mu\partial^2_{\ell'\ell''}\sigma(0)_\ell\partial^2_{\ell'\ell''}\sigma(0)_\ell\bar m_{h}^\top(I_F+\delta_{\ell'',1}\cov(\theta))\bar m_{h'} m_{h}^\top(I_F+\delta_{\ell',1}\cov(\theta))m_{h'} \nonumber
\end{align}
where we kept only terms with even number of $h$ and $h'$ because of $m_{:f}^\top\mathds{1}_H=0$ and took the expectation. Consequently there are constants $\tilde c^{(4)}_{4,i}>0, \tilde c^{(4)}_{5,i}>0$ for all $L\geq 2$ independent of $\bar m$ such that
\begin{align}
(2) &= \mathds{1}_H\mathds{1}_H^\top\otimes(c_1^{(2)}I_F+c_3^{(2)}\cov\theta)\mkern-2mu\cdot\mkern-2mu(m,m) + \sum_i^n\bar m_{:\bar f_i}\bar m_{:\bar f_i}^\top\otimes\left(\tilde c^{(4)}_{4,i}I_F+\tilde c^{(4)}_{5,i}\cov\theta\right)\mkern-2mu\cdot\mkern-2mu(m, m)
\end{align}
Assembling the two parts together we obtain
\begin{align}
& \tilde{\mathcal{E}}_\sigma(m+\bar m,0,0,0) = \tilde{\mathcal{E}}_\sigma(\bar m,0,0,0)+\tilde{\mathcal{E}}_\sigma(m,0,0,0)-\tilde{\mathcal{E}}_\sigma(0,0,0,0) \\
&\quad {}+\sum_i^n\tilde c_{1,i}^{(4)}||\bar m_{:\bar f_i}||_2^2 \left(I_H\otimes\cov\theta\right) \mkern-2mu\cdot\mkern-2mu(m, m) + \sum_i^n\mathrm{Diag}(\bar m_{:\bar f_i}^{\odot 2})\otimes\left(\tilde c^{(4)}_{2,i}I_F+\tilde c^{(4)}_{3,i}\cov\theta\right)\mkern-2mu\cdot\mkern-2mu(m, m) \nonumber \\
&\quad {}+ \sum_i^n\bar m_{:\bar f_i}\bar m_{:\bar f_i}^\top\otimes\left(\tilde c^{(4)}_{4,i}I_F+\tilde c^{(4)}_{5,i}\cov\theta\right)\cdot(m, m) + \mathcal O(||\bar m||_F^4, ||m||_F^4)\ . \nonumber
\end{align}
For a particular feature $f$ and a particular magnetization $m_{:f}$ we extract the corresponding term to obtain
\begin{align}
\nabla_{m_f,m_f}^2\tilde{\mathcal{E}}_\sigma(\bar m,0,0,0) &= \nabla_{m_f,m_f}^2\tilde{\mathcal{E}}_\sigma(0,0,0,0) +
\sum_i^nc_{1,i}^{(4)}||\bar m_{:\bar f_i}||_2^2 I_H \\
& {}+\sum_i^nc^{(4)}_{2,i}\mathrm{Diag}(\bar m_{:\bar f_i}^{\odot 2}) + \sum_i^nc^{(4)}_{3,i}\bar m_{:\bar f_i}\bar m_{:\bar f_i}^\top + M+\mathcal O(||\bar m||_F^4) \nonumber
\end{align}
with $c_{1,i}^{(4)}=\var(\theta_f)\tilde c_{1,i}$, $c^{(4)}_{2,i}=\tilde c^{(4)}_{2,i}+\tilde c^{(4)}_{3,i}\var(\theta_f)$ and $c^{(4)}_{3,i}=\tilde c^{(4)}_{4,i}+\tilde c^{(4)}_{5,i}\var(\theta_f)$.
\end{proof}

\subsection{Expressiveness and performances of different activation functions}
\label{sec:BayesProofs}
In this part we derive the results about the Bayes risk and the expressivity of the softmax and softmax-1 activation functions.

\begin{proof}[Proof of Proposition \ref{res:BayesRisk}]
The Bayes-optimal estimator is the posterior mean under our probabilistic model. It is given by the conditional probability $P(\epsilon=\ell|X, \{k_f^*\}_{f\in[F]})$:
\begin{align}
\hat{y}_\mathrm{Bayes}(X, k^*) &= \sum_{\ell}^L P(\epsilon=\ell|X, k^*)X_\ell\ .
\end{align}
We can compute it using the Bayes formula:
\begin{align}
    P(\epsilon=\ell|X, \{k_f^*\}_{f\in[F]}) = \frac{P(\epsilon=\ell, X| \{k_f^*\}_{f\in[F]})}{P(X| \{k_f^*\}_{f\in[F]})}\ .
\end{align}

The numerator is computed as
\begin{align}
    P(\epsilon=\ell, X| \{k_f^*\}_{f\in[F]}) &= \int_\theta P(\epsilon=\ell, \theta, X| \{k_f^*\}_{f\in[F]})\ d\theta \\
    &= \int_\theta P(X|\epsilon=\ell, \theta, \{k_f^*\}_{f\in[F]})P(\epsilon=\ell)P(\theta)\ d\theta\ ,
\end{align}
where
\begin{align}
    P(X|\epsilon=\ell, \theta, \{k_f^*\}_{f\in[F]}) &= \frac{1}{Z}e^{-\frac{\|X_{\ell} - \sum_{f}\theta_f k^*_f\|^2_2}{2}} \prod_{\ell'\neq \ell}^L e^{-\frac{\|X_{\ell'}\|^2_2}{2}} \\ 
    &= \exp{\left(-\frac{\|\sum_{f}\theta_f k^*_f\|^2_2 - 2 X_\ell^T(\sum_{f}\theta_f k^*_f)}{2}\right)} \frac{1}{Z}\prod_{\ell'}^L e^{-\frac{\|X_{\ell'}\|^2_2}{2}} \\
&= \exp{-\frac{\|\theta\|^2_2}{2}} \exp{X_\ell^T(\hat{k}(\theta))} \cdot \frac{1}{Z}\prod_{\ell'}^L e^{-\frac{\|X_{\ell'}\|^2_2}{2}}\ .
\end{align}

And denominator is a sum over $\ell'$ of the conditional $P(\epsilon=\ell, X|\{k^*_f\}_{f\in[F]})$. Combining it all together, we get 
\begin{align}
    P(\epsilon=\ell|X, \{k_f^*\}_{f\in[F]}) = \frac{\int_\theta \exp{-\frac{\|\theta\|^2_2}{2}} \exp{X_\ell^T(\hat{k}(\theta))}P(\theta) d\theta}{\sum_{\ell'}^L \int_\theta \exp{-\frac{\|\theta\|^2_2}{2}} \exp{X_{\ell'}^T(\hat{k}(\theta))}P(\theta) d\theta}.
\end{align}
\end{proof}

\begin{proof}[Proof of Proposition \ref{res:exprSoftmax}]
We show that the softmax (and softmax-v) attention is not well specified for our data model, while the softmax-1 is well specified.

The intuition is that for softmax one head cannot be good at the same time on a spike $\hat k$ and on the opposite $-\hat k$, and has to return noise in some cases. Assume that the softmax(-v) attention is well-specified, i.e. the reparameterized loss is $\tilde{\mathcal{E}}_\sigma(m,r,0,0)\approx 0$. Take disjoint $S\subset\mathbb R^F$ and $\bar S=\{-\theta, \theta\in S\}$ such that $P_\theta(S)>0$ and $P_\theta(\bar S)>0$. Recall that by our characterization

\begin{align}
\tilde{\mathcal{E}}_\sigma(m,r,0,0) &= \mathbb E_{\epsilon,\theta,\chi,\xi}\left[\sum_\ell^L\left(\delta_{\ell,\epsilon}-\sum_h^Hv_h\sigma(\chi_h)_\ell\right)^2\right] \approx 0 \\
\chi_{h\ell} &= \sum_f^Fm_{hf}\chi_{f\ell}^*+\sum_{h'}^Hr_{hh'}\xi_{h'}\ ,
\end{align}
where $v_h=\frac{1}{H}$ for softmax or some arbitrary constants for softmax-v attention.
Since we assumed that the model is well-specified, we get $\sum_h^Hv_h\sigma(\chi_h)_\ell\approx \delta_{\ell,\epsilon}$, and summing this expressions for all $\ell$:
\begin{align}
    1 = \sum_\ell\delta_{\ell,\epsilon}\approx \sum_\ell\sum_h^Hv_h\sigma(\chi_h)_\ell = \sum_h^Hv_h\sum_\ell\sigma(\chi_h)_\ell = \sum_h^Hv_h.
\end{align}
Now, let's fix the randomness in $\xi_{h'}$ and $\chi^*_f$ and only switch between $\theta$ and $-\theta$, we will denote the switched version of the random variable ${\chi}_{h}$ $\tilde{\chi}_{h}$. Since the error should be a.s. approximately 0, in both cases, we get
\begin{align}
\label{eqApp:softmax_sum}
    \sum_h v_h\left(\frac{e^{\chi_{h\ell}}}{e^{\chi_{h\epsilon}} + Z} - \frac{e^{\tilde{\chi}_{h\ell}}}{e^{\tilde{\chi}_{h\epsilon}} + \tilde{Z}}\right) &\approx 0.
\end{align}
Consider $\ell=\epsilon$ and denote $|\theta_h| = |\sum_f m_{hf}\theta_f|$ and $Z' = Z\exp{(-\sum_{f}m_{hf}(\chi^*_{f\epsilon}-\theta_f) - \sum_{h'}r_{hh'}\xi_h')}$, we get
\begin{align}
    \sum_h v_h\left(\frac{e^{|\theta_h|}}{e^{|\theta_h|} + Z'} - \frac{e^{-|\theta_h|}}{e^{-|\theta_h|} + Z'}\right).
\end{align}
It is left to notice that when $|\theta_h|$ is bounded away from 0, the difference $\left(\frac{e^{|\theta_h|}}{e^{|\theta_h|} + Z'} - \frac{e^{-|\theta_h|}}{e^{-|\theta_h|} + Z'}\right)$ is also bounded from zero by some constant $c_{|\theta|}$ for all $h\in[H]$, therefore taking all possible values of $|\theta_h|$ over the set $S$ we can find such constant $c_{|\theta|}$, that it bounds the difference. Combining it together with $\sum_h v_h\approx 1$, we get
\begin{align}
    \sum_h v_h\left(\frac{e^{|\theta_h|}}{e^{|\theta_h|} + Z'} - \frac{e^{-|\theta_h|}}{e^{-|\theta_h|} + Z'}\right) &\approx c_{|\theta|} > 0.
\end{align}
We obtain a contradiction with \ref{eqApp:softmax_sum}.

Instead, the softmax-1 is well specified, because it allows a head not to return noise in case it is not well aligned with any token. Consider a large signal $||\theta||_2>B$, $B\to\infty$, and take $H=2$ opposite heads $m_1\in\mathbb R^F$ and $m_2=-m_1$. Assume that $m_1$ is chosen such that the hyperplane $m_1^\top\theta=0$ has a null probability. Take $b_1=b_2=B^{\sfrac{3}{2}}$ the biases of the softmax-1, $v=H$, and scale $m_1$ and $m_2$ as $B$. By symmetry assume that $m_1^\top\theta>0$. Then one has the scalings
\begin{align}
\chi_{1\epsilon} &= m_1^\top\theta+\mathcal O(B)=\Theta(B^2) & \chi_{1,\ell\neq\epsilon} &= m_1^\top\chi_\ell^*+\mathcal O(1)=\Theta(B) \\
\chi_{2\epsilon} &= m_2^\top\theta+\mathcal O(B)\ll -1 & \chi_{2,\ell\neq\epsilon} &= m_2^\top\chi_\ell^*+\mathcal O(1)=\Theta(B)
\end{align}
Consequently the attention scores are
\begin{align}
\sigma(\chi_1)_\epsilon &= \frac{e^{\Theta(B^2)}}{e^{\Theta(B^{\sfrac{3}{2}})}+e^{\Theta(B^2)}+e^{\Theta(B)}}\to 1 & \sigma(\chi_1)_{\ell\neq\epsilon} &= \frac{e^{\Theta(B)}}{e^{\Theta(B^{\sfrac{3}{2}})}+e^{\Theta(B)}}\to 0 & \sigma(\chi_2)_\ell &\to 0 \\
&\frac{1}{H}\sum_h^H\sigma(\chi_h)_\epsilon \to 1 & &\frac{1}{H}\sum_h^H\sigma(\chi_h)_{\ell\neq\epsilon} \to 0
\end{align}
and $\tilde{\mathcal{E}}_\sigma(m,r,b,v)=0$.

\end{proof}

\newpage

\section{Dynamical phase transition}
\label{secApp:phaseTransition}
We provide a heuristic argument and numerical evidence for the presence of the dynamical phase transition stated in Conjecture~\ref{conj:phase}. The specialization dynamics is controlled by the escape from the unspecialized saddle point of the loss. Having $m(\tau^\mathrm{u})$ the alignment at the unspecialized saddle point, expanding $m$ around it, according to Prop.~\ref{res:sgd} the dynamics is
\begin{align}
\frac{\partial}{\partial\tau}m(\tau) &= -\nabla_m\tilde{\mathcal{E}}_\sigma(m(\tau)) \\
&= -\nabla^2_{m,m}\tilde{\mathcal{E}}_\sigma(m(\tau^\mathrm{u}))\cdot(m(\tau)-m(\tau^\mathrm{u}))\ .
\end{align}
We integrate it to obtain the exponential escape
\begin{align}
m(\tau) &= e^{-(\tau-\tau^\mathrm{u})\nabla^2_{m,m}\tilde{\mathcal{E}}_\sigma(m(\tau^\mathrm{u}))}\cdot m(\tau^\mathrm{u})\ .
\end{align}
We now consider the eigendirection $s\in\mathbb R^{H\times F}$ of minus the Hessian $-\nabla^2_{m,m}\tilde{\mathcal{E}}_\sigma(m(\tau^\mathrm{u}))$ with most positive eigenvalue $c$. As we show in Lemma~\ref{resApp:hessI}, $s$ corresponds to a direction of specialization of the heads. We call $m_s$ the projection of $m$ onto $s$. The dynamics is then, up to terms that take longer to grow
\begin{align}
m_s(\tau) &= e^{c(\tau-\tau^\mathrm{u})}m_s(\tau^\mathrm{u})\ .
\end{align}
For $\delta>0$ small enough (independently of $D$, $\eta$) we define the specialization time $\tau^\mathrm{s}$ to be such that $||m_s(\tau^\mathrm{s})-m_s(\tau^\mathrm{u})||_F=\delta$. We have
\begin{align}
\tau^\mathrm{s} &= \tau^\mathrm{u}+\frac{1}{c}\left(\log(\delta+||m_s(\tau^\mathrm{u})||_F)-\log ||m_s(\tau^\mathrm{u})||_F\right)\ .
\end{align}
The initial specialized component $m_s$ is a random variable that depends on the random initialization of~$k^\star$. We consider the limits $D\to\infty$, and later $\eta\to 0$. Its amplitude scales like $m_s=D^{-\sfrac{1}{2}}\eta\tilde m_s\to 0$ with $\tilde m_s=\Theta_{D,\eta}(1)$. We rescale the time as $\tilde\tau=\tau\log(\sqrt D/\eta)^{-1}$. According to Prop.~\ref{res:1stPhase} the unspecialized time is constant $\tau^\mathrm{u}=\Theta_{D,\eta}(1)$. We obtain
\begin{align}
\tilde\tau^\mathrm{s} &= \frac{1}{c}+\frac{1}{c\log(\sqrt D/\eta)}\left(c\tau^\mathrm{u}+\log\delta-\log ||\tilde m_s(\tau^\mathrm{u})||_F\right)\ . \label{eq:tildeTauSpe}
\end{align}
This expression shows that $\tilde\tau^\mathrm{s}$ concentrates to a deterministic quantity $1/c$ and that the specialization transition is sharp at $\tilde\tau^\mathrm{s}$. Indeed, in the limit $D\to\infty$ we have that $\tilde\tau^\mathrm{s}\to 1/c$. $c$ is a random quantity that depends on the realization of the data and of $k^*$; its variance is $1/\sqrt D$ and thus it concentrates. Moreover, $\tilde\tau^\mathrm{s}$ does not depend on $\delta$ in the leading order and the transition is thus sharp.

We can consider a more common scenario where the dataset is given and $D$ is fixed. In this case, one can still take the limit of small initialization $\eta\to 0$. The same analysis leading to Eq.~\eqref{eq:tildeTauSpe} holds. The transition is still sharp at $\tilde\tau^\mathrm{s}=1/c$ but $c$ now admits random fluctuations from one realization of the data to another. We illustrate this in Fig.~\ref{fig:phaseTransition} for the flipping spike at $F=H=2$. We consider several independent realizations of the data and the SGD. We show that for various $D$, for $\eta$ going to 0, the empirical means of $\tilde\tau^\mathrm{s}$ collapse to the same value as $\log(\sqrt D/\eta)\to 0$. The empirical variance of $\tilde\tau^\mathrm{s}$ decreases as $\log(\sqrt D/\eta)\to 0$ but go to 0 only when $D\to\infty$.

\begin{figure}[h!]
    \centering
    \includegraphics[width=0.9\linewidth]{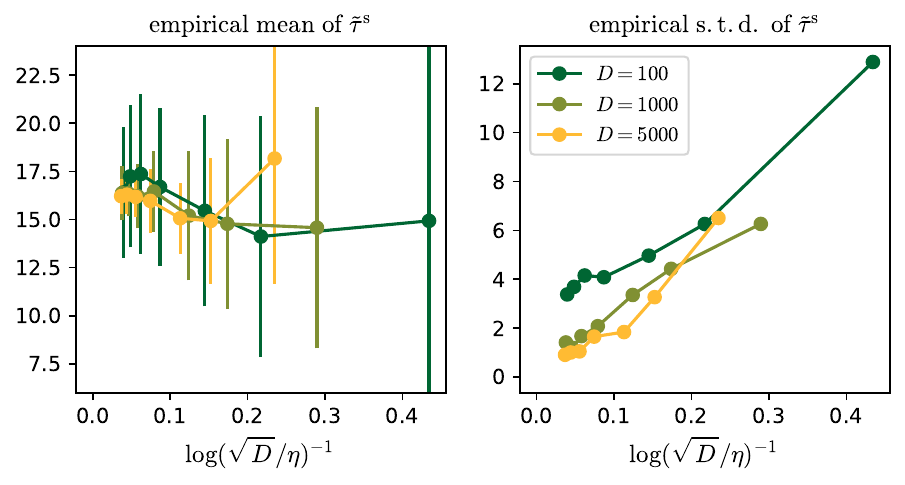}
    \caption{Phase transition: concentration of the specialization time $\tilde\tau^\mathrm{s}$. Numerical simulations of SGD. We consider sequence length $L=5$, $H=2$ heads with $\sigma$ softmax attention, $F=2$ features, $\theta$ drawn from the flipping spike distribution, with signal strengths $\nu_1=\nu_2=2$. The threshold to determine the specialization is $\delta=0.2$. The means and variances are empirically computed over 100 independent runs for each point.}
    \label{fig:phaseTransition}
\end{figure}

We show that the insights of our analysis extend to semi-realistic data, considering sequences based on the MNIST detection task described in Appendix~\ref{secApp:mnist}. $D=\tilde D=784$ is fixed and we take the limit of small initializations $\eta\to 0$. We define each feature/direction $f$ to be the average digit $f$, for $f=0,\ldots,9$. In Fig.~\ref{fig:phaseTransition_mnist} we see that the rescaled specialization times for each feature concentrate to a deterministic value, independent of the run and the initial condition, different for each $f$.

\begin{figure}[h!]
    \centering
    \includegraphics[width=0.9\linewidth]{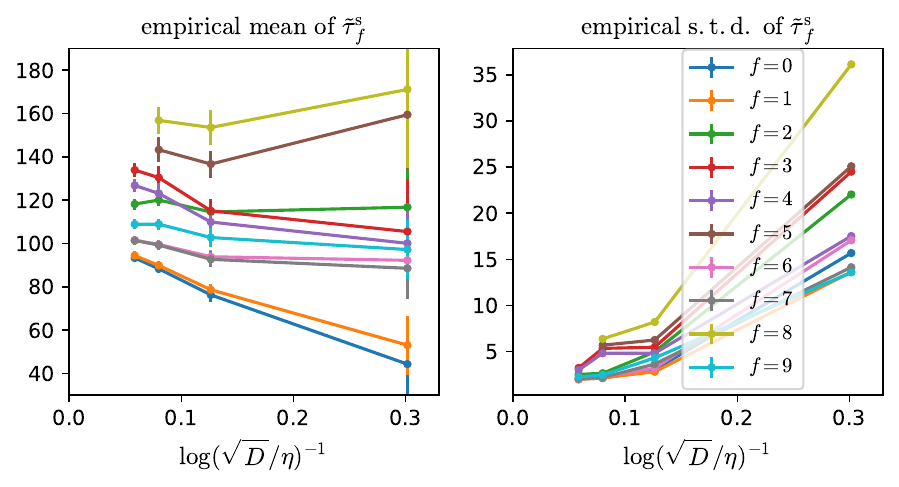}
    \caption{Phase transition: concentration of the specialization time $\tilde\tau_f\tau^\mathrm{s}$. Numerical simulations of SGD on the MNIST detection task App.~\ref{secApp:mnist}. $H=10$ heads with $\sigma$ softmax attention. The threshold to determine the specialization time $\tau_f^\mathrm{s}$ in each direction $f$ is $||m_{:,f}-H^{-1}\sum_h^Hm_{h,f}\mathds{1}_H||_2>1$. The means and variances are empirically computed over 100 independent runs for each point.}
    \label{fig:phaseTransition_mnist}
\end{figure}

\clearpage
\newpage

\section{MNIST digits detection}
\label{secApp:mnist}
To extend the scope of our theoretical predictions to the case of more complex data, we conducted experiments where the model is trained to detect a handwritten digit from a sequence of patches where the other patches are pure noise. For each sequence we take the relevant token (patch) to be uniformly sampled from the MNIST dataset \cite{deng2012mnist}. We denote the dataset $\mathcal S=\{\tilde X_\mu\in\mathbb R^{\tilde D}\}_{\mu\in[\tilde N]}$, $\tilde N=10^4$, $\tilde D=784$. The MNIST dataset is normalized in the following way: the values of the pixels are first rescaled as $\tilde X_\mu\mapsto\tilde X_\mu/255\in[0,1]$. The mean direction is then removed according to $\tilde X_\mu\mapsto\tilde X_\mu-\tilde N^{-1}\sum_{\mu'}^{\tilde N}\tilde X_{\mu'}$. Each sequence is constructed as
\begin{align}
X_\ell\sim\mathcal N(\nu\delta_{\ell,\epsilon}\hat k,I_D)\ ,\qquad \hat k\sim\mathrm{unif}(\mathcal S)\ .
\end{align}
The number of patches is $L=5$; the signal strength is $\nu=0.3$. The attention model is the same as in the main eq.~\eqref{eq:attention_model}.

\section{Numerical simulation details}
\label{secApp:numerical}
For the reparametrized population loss $\tilde{\mathcal E}$ estimation we discretize the flow with a step size $\delta=0.02$. We use Monte Carlo integration with $10^5$ samples to compute the expectations. We keep the same samples for each step of the gradient descent, which doesn't affect the behavior of the model as can be seen in Fig. \ref{6_fig:app_dif_magnetization}.

For the initialization, when not comparing to SGD, we add initial noise $\mathcal N(0,10^{-4})$ to $m$ and $r$, which allows to break the initial symmetry in the parameters. When comparing to SGD we initialize $m$ and $r$ to their empirical values. For SGD and the loss $\mathcal E$, we take $N_b=D$ and learning rate $\gamma=0.02$.

We provide the code to run our predictions in the supplementary material. Running one gradient descent takes a few minutes to half an hour and a few GBs on a local GPU. Overall we ran a thousand of descents.

\begin{figure}[h!]
    \centering
    \includegraphics[width=0.5\linewidth]{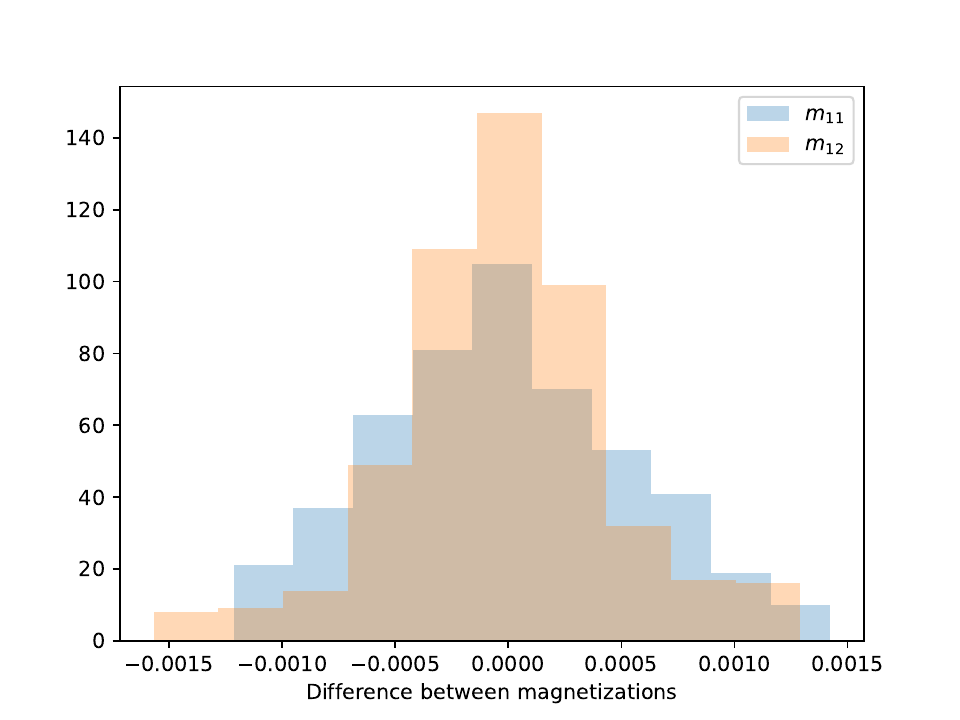}
    \caption{Histogram of the differences between magnetizations of the first head in the model trained with changed MC samples and the same MC samples at each step of the training. The data are sampled from the flipping spike distribution, $H=2, F=2, \nu_1=\nu_2=2, L=8$.}
    \label{6_fig:app_dif_magnetization}
\end{figure}